\documentclass{article}

\usepackage{fullpage}
\usepackage{graphicx}
\usepackage{subcaption}
\usepackage{algpseudocode}
\usepackage[ruled]{algorithm2e}
\usepackage[round]{natbib}

\usepackage{commands}

\title{Learning RUMs: Reducing Mixture to Single Component via PCA}
\date{}
\author{
	Devavrat Shah\\
	\texttt{devavrat@mit.edu}
	\and
	Dogyoon Song\\
	\texttt{dgsong@mit.edu}
}

\begin{document}



%\twocolumn[
%
%\aistatstitle{}
%%Reducing Mixture Learning of RUM to Single Component Learning using PCA}
%%\aistatstitle{Reducing Mixture of Random Utility Models Using PCA}
%
%%\aistatsauthor{ Devavrat Shah \And Dogyoon Song }
%%
%%\aistatsaddress{ MIT LIDS/IDSS} 
%]

\maketitle

\begin{abstract}
We consider the problem of learning a mixture of Random Utility Models (RUMs). Despite the success of RUMs in various domains and the versatility of mixture RUMs to capture the heterogeneity in preferences, there has been only limited progress in learning a mixture of RUMs from partial data such as pairwise comparisons. In contrast, there have been significant advances in terms of learning a single component RUM using pairwise comparisons. In this paper, we aim to bridge this gap between mixture learning and single component learning of RUM by developing a `reduction' procedure. 
We propose to utilize PCA-based spectral clustering that simultaneously `de-noises' pairwise comparison data. We prove that our algorithm manages to cluster the partial data correctly (i.e., comparisons from the same RUM component are grouped in the same cluster) with high probability even when data is generated from a possibly {\em heterogeneous} mixture of well-separated {\em generic} RUMs. Both the time and the sample complexities scale polynomially in model parameters including the number of items.
Two key features in the analysis are in establishing (1) a meaningful upper bound on the sub-Gaussian norm for RUM components embedded into the vector space of pairwise marginals and (2) the robustness of PCA with missing values in the $L_{2, \infty}$ sense, which might be of interest in their own right. 

%
%In particular, the only known statistical guarantees are restricted to a mixture of specific RUMs and there are no comparable results known so far for a mixture of generic RUMs.  In this paper, we address the challenge by 
%
%proposing to use spectral clustering for identifying RUM components and analyzing the statistical and computational performance of the proposed procedure. Specifically, we show that spectral clustering 
%successfully cluster samples from a possibly {\em heterogeneous}\footnote{Components are modeled by different types of RUMs.} mixture of {\em generic} RUM components with high probability from $\texttt{poly}(n, r, 1/\Gamma, \tau)$ number of pairwise comparisons where $n$ is the number of items, $r$ is the number of components, $\Gamma$ is the gap between component centers and $\tau$ is the sub-Gaussian norm of RUM components. 
%Two key features in the analysis is in establishing (1) a meaningful upper bound on the sub-Gaussian norm for RUM components embedded into the vector space of pairwise marginals and (2) the robustness of PCA with missing values in the $L_{2, \infty}$ sense, which might be of interest in their own right. 

\end{abstract}

%%% ======================================================================
%%%	Introduction
%%% ======================================================================

\section{INTRODUCTION}\label{sec:introduction}

Random Utility Model (RUM), originally proposed by \cite{thurstone1927law}, has since been utilized across disciplines to capture preferences 
including social choice in Economics, policy making in Political Science, revenue management in Operations Research, and ranking 
(or recommendations) in Computer Science and Statistics to name a few. RUM is a family of models for distribution over permutations 
parameterized by inherent ``skills'' or ``scores'', say $u_a \in \Reals$ for each item (or option or player, etc.) $a \in [n]$ amongst the universe 
of $n$ items. A random permutation as per this model is generated by sampling random variables $Z_a, ~\forall a\in [n]$, where
$Z_a = u_a + \varepsilon_a$, with noise $\varepsilon_a$; then the ordering of $Z_a$ provides the random permutation. 
Two popular instances of RUM are obtained by selecting noise distribution as Gumbel and Gaussian, respectively: Gumbel leads to the so-called 
Multinomial Logit (MNL) model that is used for transportation applications by Nobel Laureate \cite{McF}; % or airline ticket pricing, cf. \cite{pricing}; 
Gaussian leads to a model popularly utilized for rating online players through TrueSkill rating, cf. \cite{trueskill}. 

In practice, a mixture of RUMs rather than a simple RUM has been effectively utilized to model the diversity in preferences; for example, 
see \cite{pricing} for airline ticket pricing application. Each component of the mixture RUM represents a group of individuals with similar 
preferences and there are multiple such clusters in the mixture. 

Learning the underlying RUM or the mixture of RUMs from partial information on preferences such as pairwise comparisons is a natural question; 
for example, learning preferences for ticket prices using comparative shopping data for airlines. Despite its importance, systematic studies toward 
developing computationally efficient methods with rigorous statistical guarantees have occurred only in the last decade or so. 
While some good progress has been made for learning single component RUMs or homogeneous mixtures of specific RUMs, there is little or 
no improvement toward learning a {\em heterogeneous}\footnote{Components are modeled by different types of RUMs, e.g., some components 
are MNL and others are Gaussian.} mixture of RUMs which provide flexibility to achieve a more parsimonious (sparser) representation as well as 
robustness against model mismatches. Our aim in this work is at learning such a model from pairwise comparison data.

%While good progress has been made for learning single component RUMs
%and some progress for homogenous mixtures of RUMs, a little or no progress has been for learning heterogenous mixture RUMs
%which is the focus of this work. 
%learning mixture of 
%
%Notably, a significant progress has been made for learning a single MNL model using 
%pairwise comparisons, cf. \cite{lu2011learning, negahban2012iterative, soufiani2013generalized} and its variations cf. \cite{jagshah2008}. 
%On the other hand, the progress on learning mixture of MNL has been relatively limited, cf. \cite{ammar2014s, oh2014learning}. 
%Moreover, rigorous results are not known for learning a mixture of RUM with generic noise model, e.g., a heterogeneous mixture containing 
%some MNL components and some Gaussian RUM components. See Section \ref{ssec:related} for details on related work. 
%\DG{It could be a good idea to add a sentence explaining why one should consider a heterogeneous mixture.}

%\subsection{Contributions}\label{ssec:contributions}
\subsection{Our Contributions}
%\paragraph{Our Contributions.} 
As the main contribution of this work, we argue that spectral methods (e.g., PCA) can correctly classify 
pairwise comparison data generated from a heterogenous mixture of generic RUMs. We can first separate a collection of pairwise
comparison data into clusters and then learn the parameters for single RUMs describing each of the clusters. That is, the hard task of 
learning a mixture RUM reduces to (relatively) simpler tasks of learning single RUMs. 

We present Algorithm \ref{alg:main_algorithm} as an instance of the proposed procedure. It requires pairwise comparisons that scale 
polynomially in model parameters $n$, $k$, $1/\Gamma$ and $\tau$ where $n$ is the number of options, $k$ is the number of mixture 
components, $\Gamma$ is the minimum gap between the centers of mixture components and $\tau$ is the maximum sub-Gaussian norm 
of RUM components. The algorithm does not require any knowledge of the noise model within each component RUMs, thus being robust 
against model mismatches. See Theorem \ref{thm:main} and Corollary \ref{cor:main} for precise statements of our results.

Specifically, we view a permutation over $n$ items as a vector embedded into $\{ \pm \frac{1}{2} \}^{n \choose 2}$ rather than a combinatorial 
object. Each coordinate of the binary vector represents the outcome of pairwise comparisons, i.e., which item in the pair precedes the other. 
From matrix estimation viewpoint, we stack partial observations about $N$ permutations through (a small subset of) pairwise comparisons 
to construct a matrix of size $N \times {n \choose 2}$ that might have lots of missing values (unobserved comparisons). After filling $0$ in place 
of missing comparisons, we perform PCA to `denoise' the vectors and then group them using single-linkage clustering (e.g., Kruskal) based on 
pairwise Euclidean distance. The proposed algorithm is closely related to that in \cite{achlioptas2005spectral}. 

Our analysis is based on the observations that each row of the data matrix is a random vector from a sub-Gaussian distribution and that 
the expectation of the random matrix has rank $r$, which is equal to the dimension of the span of the component centers. Hence, singular value 
thresholding (SVT) on this random matrix (with an appropriate choice of threshold) results in the projection onto top $r$ principal components 
with the noise power diminished. In the end, we prove that all the data points get correctly clustered from $O(\sqrt{r}\tau n^{6} \log n / \Gamma )$ 
number of pairwise comparisons.

With $\tau=O(\sqrt{n})$ and $\Gamma=\Omega(n)$ for typical setups, this result implies that a mixture of $r$ \emph{generic} RUM 
components\footnote{Our result extends to the case where there are $k \geq r$ components as long as the mean matrix has rank $r$.} can be 
learnt using $O\big(r^{0.5}n^{5.5} \log n\big)$ number of comparisons; see Remark \ref{rem:number_comparisons}. 
For example, having $O \big( \beta r^{0.5} n^{5.5} \log n / \min_{a \neq b} | u_a - u_b | \big)$ comparisons suffices for clustering a mixture MNL, 
and $O ( \sigma r^{0.5} n^{5.5} \log n )$ is sufficient for a mixture of RUMs with Gaussian noise\footnote{The quantities $u_i$, $\beta$, $\sigma$ 
are defined in Section \ref{sec:examples}.}, cf. Section \ref{sec:examples}.

We make a few remarks about our result. When restricted to a homogeneous mixture of MNLs, our result might seem weaker than the result by 
\cite{oh2014learning}, which states that a mixture of $r$ MNL components can be learned using a total of $O\big(r^{3.5}n^3 \log^4 n\big)$ 
pairwise comparisons when $r \ll n^{2/7}$ and the model parameters are sufficiently incoherent. However, the result by \cite{oh2014learning}
works for mixture of MNLs only, while our algorithm does not exploit such a priori knowledge. Further, our sample complexity outperforms theirs 
when $r \approx n$ and moreover, our algorithm turns out to require only $O\big(n^4\big)$ comparisons in a sufficiently nicer setup of mixture MNLs.
%Like ours, \cite{oh2014learning} cluster data into components first and then utilize the single
%component learning algorithm. 

Technically, \cite{oh2014learning} require independence on pairwise comparisons from the same permutation instance. In the language of matrix
estimation view described above, they require observations across columns in a given row to be independent, whereas we do not impose such 
an artificial condition and our method works with dependence across columns for a given row. This is where our sharp characterization of 
the sub-Gaussian property of RUM as in Definition \ref{def:subG} comes handy. 

%Lastly, we emphasize two features of our result. First, our result is applicable to a general mixture of RUMs; our algorithm identifies mixture 
%components with the aforementioned sample complexity regardless of the noise model for RUM components, even when each component 
%is described by different types of RUMs. To the best of our knowledge, there is no prior work that attempts to identify components of mixture 
%RUM in the generality considered in our work. Second, our (simplified version of) sample complexity bound might seem not as strong as 
%that of \cite{oh2014learning}, but it is likely caused by simplification steps in our analysis. For example, we show that $O \big( n^4 \big)$ 
%is sufficient under a nice setup, cf. Remark \ref{rem:samples_necessary}.

Finally, we note that our analysis implies the robustness of PCA in the sense that SVT on a random matrix leads to faithful estimation of the underlying 
low-rank mean matrix in the max row $\ell_2$ norm (a.k.a. $L_{2,\infty}$ matrix norm)  sense. This result is stronger than the widely used notion of robustness 
in Frobenius norm sense\footnote{$L_{2,\infty}$ vs Frobenius $\approx$ worst case vs average case.} and is likely to be of interest in its own right.

%%% ======================================================================
%%%	Related Works
%%% ======================================================================

\subsection{Related Work}
%\paragraph{Related Work.}
%\noindent{\bf Related Work.} %\label{ssec:related} 
%\subsection{Related Work}\label{ssec:related}
% A brief overview of related works follow. 
% A brief overview follows.

%\medskip
\paragraph{RUM and RUM Learning.}
%\noindent{\bf RUM and RUM Learning.}
%Random Utility Model (RUM) 
RUM is a broad class of models to describe ranking data. A special case of RUMs is the Multinomial Logit (MNL), 
which can be seen as the extension of multinomial logistic regression to ranking data. The model is also known as the Bradley-Terry model 
or the Plackett-Luce model; cf. \cite{bradley1952rank, r1959individual, plackett1975analysis}. RUM with Gaussian noise is widely used 
in skill-based ranking systems cf. \cite{trueskill}.
%There are also other popular ranking models including the Mallows model \cite{mallows1957non} and sparse choice model \cite{jagshah2008}. 

In the context of learning RUMs, various methods have been studied to estimate parameters of a simple MNL model from pairwise comparison data, e.g., \cite{lu2011learning, negahban2012iterative, soufiani2013generalized}. However, there have been only limited 
advances toward learning mixtures of such models. While some provable methods have been proposed by \cite{ammar2014s, oh2014learning, 
awasthi2014learning, sedghi2016provable}, these methods have certain limitations. Since these methods are essentially methods of moments 
that exploit tensor decomposition, they are model-specific and require certain technical assumptions, e.g., tensor non-degeneracy. In a sense, these methods solve a system of equations between the model parameters and the empirical moments. Therefore, it is unlikely that these methods can handle a heterogeneous mixture of RUMs. In contrast, our approach only depends on inequality-type certificates\footnote{based on the light-tail behavior of each component RUM.} and is agnostic to model specifications; thus 
it can deal with even heterogeneous mixtures. 

Recently, \cite{zhao2018learning} discuss identifiability of mixture RUMs of a fixed utility distribution and propose algorithms based on 
EM and generalized method of moments, however, they do not provide rigorous analysis for time and sample complexity.

%\subsection{Further Related Work}\label{ssec:further.related}
%
%Our work broadly falls in the category of learning mixture distributions. Learning mixtures has a long history, 
%dating back to  \cite{pearson1894contributions}. There are three widely utilized approaches: spectral clustering,
%methods of moments and sum-of-squares certificates. 

\paragraph{Spectral Clustering.} 
%\noindent{\bf Spectral Clustering.} 
This line of works generally require certain separation condition between the component centers for clustering. 
Since \cite{vempala2002spectral} first suggested performing PCA to achieve a separation condition that does not depend on the ambient dimension 
in the context of learning a mixture of spherical Gaussians, this result has been generalized in subsequent works by \cite{achlioptas2005spectral, brubaker2008isotropic, kumar2010clustering, awasthi2012improved}. For example, \cite{kumar2010clustering} -- further improved by 
the subsequent work of \cite{awasthi2012improved} -- show that a variant of spectral clustering (PCA followed by Lloyd's algorithm) can 
recover the hidden components as long as the cluster centers are separated by $\widetilde{\Omega}(\sqrt{r})$ and each cluster has 
bounded covariance. This is the best known result in spectral clustering. Our work is a natural generalization of this line of approach 
where it enables spectral clustering {\em with missing data} whereas all the prior works assume that data is {\em fully} observed.

\paragraph{Method of Moments.} 
%\noindent {\bf Method of Moments.} 
Dating back to the original work by  \cite{pearson1894contributions}, the goal in this approach is to compute empirical moments from observations 
and then learn mixture parameters by solving a system of (typically non-linear) equations between moments and model parameters. This line of 
works utilize higher order moment information to obtain more equations so as to identify mixtures. \cite{anandhkumaretal} suggest to 
use tensor decomposition for learning mixture in $\Reals^d$ and \cite{oh2014learning} utilize this tensor decomposition of higher-order 
moments in the context of mixture MNL learning. This approach is robust to missing values thanks to the robustness of moments, however, 
it requires precise knowledge on the underlying model to match the moments and the model parameters. Therefore, the method is inherently 
brittle to model mismatch. 
%This is the reason why \cite{oh2014learning}, as is, does not extend for the setting of RUM with Gaussian noise, for example. 
In contrast, our work is model agnostic and can work with a mixture of generic RUM components. 

\paragraph{Sum-of-Squares Certificates.} 
%\noindent {\bf Sum-of-Squares Certificates.} 
Recently, there has been a flurry of works on mixture learning based on Sum-of-Squares (SoS) proofs and robust estimation, 
cf. \cite{diakonikolas2016robust, kothari2017better, diakonikolas2018list, hopkins2018mixture}. These methods have many virtues: 
(1) they break the $\widetilde{\Omega}(\sqrt{r})$ barrier in the separation requirement at the expense of increased sample complexity; 
and (2) they are more robust to model mismatch compared to the method of moments because these methods work with certain inequality-type 
constraints that are satisfied subject to sub-Gaussianity and certified by SoS proofs. However, as developed in literature, the up-to-date technique 
applies only to the setting of continuous distribution. In contrast, we work with discrete distributions over permutations.

%%% ======================================================================
%%%	Model
%%% ======================================================================

\section{MODEL AND PROBLEM}\label{sec:model}
%%% ======================================================================
%%%	Generative Model
%%% ======================================================================
%\subsection{Model}\label{sec:gen_model}

%\subsection{Generative Model} 
\paragraph{Generative Model. } 
%\noindent{\bf Generative Model. } 
Let $S_n$ denote the set of all $n!$ permutations (=rankings) over $n$ items.  We are interested in learning
a mixture of {\em simpler} distributions over $S_n$ from observations about these rankings through pairwise
comparisons. In particular, our interest is in learning\footnote{Here, we use the term `learning' to indicate clustering 
rather than identifying parameters for certain model.} mixture of Random Utility Models. 
\begin{definition}[Random Utility Model]\label{def:rum}
Each item $a\in [n]$ has a deterministic utility (skill, score) $u_a \in \Reals$ associated with it. 
The random utility $Z_a$ for $a \in [n]$ obeys the form $Z_a= u_a + \eps_a$, where $\eps_a$ are 
independent random variables. Once $Z_a$ are generated for all $a \in [n]$, $Z_1, \ldots, Z_n$ are 
sorted in descending order and the order of indices is output as a permutation.
\end{definition}
The distribution of random noise $\eps_a, ~a \in [n]$ in RUM induces a distribution of permutations in $S_n$. 
Different noise models lead to different types of distributions in $S_n$. We describe two popular examples next. 
\begin{example}[Multinomial Logit Model]\label{example:mnl}
Let $\eps_a$ be drawn as per the Gumbel distribution with mode $0$ and scaling parameter $\beta > 0$, 
i.e., the PDF of $\eps_a$ is %given by
\[	f(x) = \frac{1}{\beta} \exp \bigg( - \Big( \frac{x }{\beta} + 
	\exp \Big( - \frac{x}{\beta} \Big) \Big) \bigg), \quad\text{for}~ x \in \Reals.	\]
The resulting distribution over $S_n$ is known as the Multinomial Logit (MNL) Model. An important property
of the distribution is that for all $a \neq b \in [n]$,
\begin{equation}\label{eqn:prob_mnl}
	\Prob{Z_a  \geq Z_b}  = \frac{\exp\big(\frac{u_a}{\beta}\big)}{\exp\big(\frac{u_a}{\beta}\big) + \exp\big(\frac{u_b}{\beta}\big)}.
\end{equation} 
\end{example}
\begin{example}[RUM with Gaussian Noise]\label{example:gaussian}
For $a \in [n]$, let $\eps_a$ be drawn as per Gaussian distribution with mean $0$ and variance $\sigma_a^2 > 0$, 
i.e., the PDF of $\eps_a$ is given by 
\[	f(x) = \frac{1}{\sqrt{2\pi \sigma_a^2}} \exp \left( - \left( \frac{x^2 }{\sigma_a^2} \right) \right), \quad\text{for}~ x \in \Reals.	\]
We shall call the resulting distribution over $S_n$ as RUM with Gaussian noise. 
\end{example}
%We are interested in a finite mixture of simpler distributions over $S_n$. 
Let there be $k \geq 1$ mixture components, each of which is an instance of a certain RUM. For $i \in [k]$, let the $i$-th 
mixture component be parameterized by true utilities $u^{(i)} = (u_a^{(i)})_{a \in [n]} \in \Reals^n$ and noise parameters 
$\gamma^{(i)} = ( \gamma_a^{(i)} )_{a \in [n]} \in \Reals^{n \times D}$ where the noise distribution is determined by $D \geq 1$ 
dimensional parameter. For instance, $\gamma^{(i)} = ( \beta_a )_{a \in [n]} \in \Reals^n$ in the MNL example and 
$\gamma^{(i)} = ( \sigma_a^2 )_{a \in [n]} \in \Reals^n$ for the Gaussian RUM.

We consider the following hierarchical data generation process to sample a permutation from the mixture:
%\begin{enumerate}\vspace{-1em}\setlength\itemsep{0em}
%	\item
First,  choose a mixture component $i \in [k]$ at random with probability $p_i, i \in [k]$; and then 
%	\item
%Then 
	draw a permutation per the RUM with parameters $(u^{(i)}, \gamma^{(i)})$. 
%\end{enumerate}
%\subsection{Model for Observations} 

\paragraph{Model for Observations.} 
%\noindent{\bf Model for Observations.} 
Let $\iota: S_n \to \big\{ \pm \frac{1}{2} \big\}^{n \choose 2}$ be a representation of permutation 
$\perm \in S_n$ through pairwise comparisons such that for each $1 \leq a < b \leq n$,
\begin{align*}%\label{eqn:pairwise_representation}	
	\iota(\perm)_{(a, b)} %& =\frac{1}{2}	\Big( \Ind{ \perm^{-1}(i) \geq \perm^{-1}(j) } - \Ind{ \perm^{-1}(i) < \perm^{-1}(j) } \Big), ~~\mbox{for}~1\leq i< j \leq n. 
		& = \Ind{ \perm^{-1}(a) \geq \perm^{-1}(b) } - \frac{1}{2}.%, ~~\mbox{for}~1\leq i< j \leq n. 
\end{align*}
For notational convenience, we index the coordinates of the vector $\iota(\perm)$ using tuples $(a, b), ~1\leq a < b \leq n$. 

We assume the following model for observation of pairwise comparisons. Suppose that $N$ permutations, $\perm^\ell, ~\ell \in [N]$, 
are sampled as per a mixture of $k$ RUMs. Let $X^\ell = \iota(\perm^\ell) \in \{\frac{1}{2}, - \frac{1}{2}\}^{n (n-1)/2}$ denote the embedding 
of $\perm^\ell$ for each $\ell \in [N]$. We observe $\phi^{\ell}$ in the form of $Y^\ell$ where for each $m = (i, j)$, $~1\leq i< j \leq n$,
\begin{align}\label{eq:YnX}
Y^\ell_m & = \begin{cases} X^\ell_m & \text{~with~probability~} p \\
\star & \text{~otherwise.}
\end{cases}
\end{align}
Here, $\star$ indicates the lack of knowledge about the value. In other words, we assume that for each permutation $\phi^\ell$, we have 
access to each of ${n \choose 2}$ pairwise comparisons at random with probability $p$.
Therefore, we expect to observe approximately $pNn(n-1)/2 \approx pNn^2/2$ pairwise comparisons in total.

\paragraph{Notation.}
%\noindent{\bf  Notation.} 
Let $X \in \{-\frac12, \frac12\}^{N\times {n \choose 2}}$ be the random matrix obtained by stacking $N$ random permutations $X^\ell, ~\ell \in [N]$. 
Given $X$, let $\alpha^*(\ell) \in [k]$ denote the index of the component which $\phi^{\ell}$ is drawn from. Accordingly, we can define a map 
$\alpha^* = \alpha^*_X: [N] \to [k]$. Note that $\alpha^*$ is an imaginary map and we don't have access to it. Let $M = \Ex{X |~ \alpha^*}$ 
%$\in \big[-\frac12, \frac12\big]^{N \times {n \choose 2}}$ 
be the mean matrix where the expectation is taken with respect to the randomness 
in step 2 of the data generation process only, i.e., it is the expectation conditioned on $\alpha$. 
We let $r := \rank(M)$ and $\sigma_1(M) \geq \ldots \geq \sigma_r(M)$ denote the $r$ nonzero singular values of $M$ in the nonincreasing order. 
It is easy to observe that $r \leq k$ as $M$ has at most $k$ different types of rows.

%\subsection{Problem Statement} 
\paragraph{Problem Statement.}
%\noindent{\bf Problem Statement.}
Our objective is to identify data points with the mixture components which they are from (up to permutation). Given observations 
$Y^1, \ldots, Y^N \in \{-\frac12, \star, \frac12\}^{n \choose 2}$, we want to estimate the number of mixture components and the label map 
that minimizes the misclassification error rate, i.e., we want 
\[	\big( \hat{k}, \hat{\alpha} \big) \in \arg\min_{ k' \in \Ints \atop \alpha \in [k']^N } \left\{ \min_{\sigma \in S_{k'}} 
		\frac{1}{N} \sum_{\ell=1}^N \Ind{ {\alpha}(\ell) \neq \alpha^*(\sigma(\ell)) } \right\}.	\]
We also want to know the condition when the misclassification error rate is $o(1)$ with probability $1 - o(1)$.

%%% ======================================================================
%%%	Algorithm
%%% ======================================================================

\section{ALGORITHM}\label{sec:algorithm}
 Our algorithm is summarized in Algorithm \ref{alg:main_algorithm}. An appropriate choice of 
 parameters $t_1, t_2$ leads to successful clustering with high probability.
\begin{algorithm} [h]
	\SetKwInOut{Input}{Input}
	\SetKwInOut{Output}{Output}
	
	\Input{$( Y^1, \ldots, Y^N; t_1, t_2 )$}
	\Output{$(\hat{k}, \hat{\alpha})$}
	
	\begin{algorithmic}[1]
		\item Stack $Y^1, \dots, Y^N$ to form $Y \in \{-\frac12, \star, \frac12\}^{N \times {n \choose 2}}$.
		\item Replace $\star$ with $0$ in $Y$.
		\item Take SVD of $Y$: $Y  = \sum_{j=1}^{\min\{N, {n \choose 2}\}} \sigma_j(Y) u_j v_j^T$.
		\item Singular value thresholding at $t_1$: 
			\begin{equation}\label{eq:mhat}
				\hM = \sum_{j=1}^{\min\{N, {n \choose 2}\}} \sigma_j(Y) \Ind{ \sigma_j(Y) > t_1 } u_j v_j^T.
			\end{equation}
		\item Apply single-linkage clustering on the rows of $\hM$ with threshold $t_2$.
	\end{algorithmic}
	\caption{Clustering RUM Components}
	\label{alg:main_algorithm}
\end{algorithm}
%
% (e.g., single-linkage agglomerative clustering)
%
%if $\hM^a$ and $\hM^b$ are such that $\| \hM^{\ell_1} - \hM^{\ell_2} \|_2 \leq t$, then they are part of same cluster for any 
%$\ell_1, \ell_2 \in [N]$. The choice of appropriate $t$ leads to successful clustering with high probability. 

\section{MAIN RESULTS}\label{sec:results}
In this section, we state the main theorem about conditions under which Algorithm \ref{alg:main_algorithm} faithfully identifies
the $k$ RUM components in the mixture. With the success of the clustering, the problem reduces to a bundle of subproblems 
to learn model parameters $(u^{(i)}, \gamma^{(i)}), ~i \in [k]$ of each single RUMs. Then one may solve the subproblems 
by imposing certain model assumptions and applying parameter learning algorithms designed for such models.

\subsection{Theorem Statement}
First, we define a measure of variability in a RUM.
\begin{definition}[sub-Gaussian norm]\label{def:subG}
Let $X \in \Reals^n$ be a random vector. The sub-Gaussian norm of $X$ is  
\begin{align*}
\tau & \triangleq \inf\Big\{\rho > 0 \textnormal{ such that }\forall \lambda \in \Reals,	\\
	&\qquad\qquad \sup_{u: \|u\|_2 = 1} \Ex{ e^{\lambda \langle u, ~X-\bbE X\rangle} } \leq \exp\Big(\frac{\lambda^2 \rho^2}{2}\Big) ~\Big\}
\end{align*}
with the convention that $\inf$ of the empty set is $\infty$. 
Given a RUM, let $\phi$ be a random permutation following the RUM and let $X = \iota(\phi)$. We call the sub-Gaussian norm of $X$ 
as the sub-Gaussian norm of the RUM.
\end{definition}
Given a mixture of $k$ RUMs over $S_n$, let $\mu^{(i)} = \bbE[ X^1| \alpha^*(1)=i ] \in \big[-\frac12, \frac12\big]^{n \choose 2}$ 
and $\tau^{(i)} > 0$ denote the center and the sub-Gaussian norm of the $i$-th mixture component for $i \in [k]$. 
We define the following quantities that will be useful to state the main results: 
\begin{align}
	\tau^\star & \triangleq \max_{i \in [k]} \tau^{(i)} \label{eq:subG}\\
	\Gamma & \triangleq \min_{i, i' \in [k] :~ i \neq i'} \| \mu^{(i)} - \mu^{(i')} \|_2, \nonumber\\%\label{eq:gap}\\
	\Delta &\triangleq C \Big\{(\sqrt{p} + p \tau^\star)\sqrt{N} + (\tau^\star + K(p)) ( n + \sqrt{n} N^{\frac14})\Big\}, 
	\label{eqn:Delta_upper}
\end{align}
where $C > 0$ is an absolute constant and $K(p) \in [0, \frac{1}{4}]$ is the sub-Gaussian norm of the centered Bernoulli random 
variable\footnote{See Equation \ref{eqn:Kp.0} in Appendix \ref{sec:missing_values} for the explicit form.}.

\begin{theorem}\label{thm:main}
Let $r = \rank (M)$. Suppose that $p > \frac{4 \Delta}{\sigma_r(M)}  \vee \frac{16 \log(Nd)}{\min(N, d)}$, $\log N = o(n)$.
If 
\begin{align}\label{eq:gap.condition}
\Gamma  & \geq C' \frac{( \tau^\star+ \frac{1}{4})  \log(nN)}{p} \bigg( \sqrt{r} + \frac{n^2 \big( \tau^\star +  \frac{1}{4} \big) }{p \sigma_r(M)}  \bigg)
\end{align}
for a sufficiently large constant $C' > 0$, then with probability at least $1 - \frac{4}{N^2} - \frac{4}{n^4}$, Algorithm \ref{alg:main_algorithm} 
with proper $t_1, t_2$ successfully clusters $N$ rankings in $k$ groups. 
\end{theorem}

\begin{remark}
The algorithmic parameters $t_1$ and $t_2$ need to be chosen properly. The singular value threshold $t_1$ is proper if 
$\Delta < t_1 < \sigma_r(M) - \Delta$ and the distance threshold $t_2$ is proper if 
$\max_{\alpha^*(i_1) = \alpha^*(i_2)}\big\| \widehat{M}^{i_1} - \widehat{M}^{i_2} \big\|_2 $$<t_2<$
$\min_{\alpha^*(i_1) \neq \alpha^*(i_2)}\big\| \widehat{M}^{i_1} - \widehat{M}^{i_2} \big\|_2 $; see Theorem \ref{thm:main_clustering_temp} 
in Appendix for precise statement.
\end{remark}

\subsection{Simplification for Interpretation}

\paragraph{Balanced Mixture and the Eigengap.}
We want a lower bound on $\sigma_r(M)$. For the purpose, we assume the center of RUM components, $\mu^{(i)}, ~i \in [k]$, 
are nearly orthogonal to each other\footnote{We would say $\mu^{(i)}$, $i \in [k]$ are nearly orthogonal to each other if there is a constant 
$C > 0$ (e.g., $C = 8$) such that $ \frac{ \left| \langle \mu^{(i)}, \mu^{(i')} \rangle \right| }{ \| \mu^{(i)} \|_2 \| \mu^{(i')} \|_2 } \leq \frac{ C \sqrt{\log n} }{ n }$ 
for all $i, i' \in [k]$.}. The assumption is likely to hold in the high-dimensional setting because any pair of random vectors from 
an isotropic distribution--spherical Gaussian distribution, discrete uniform distribution over the hypercube, etc.--is nearly orthogonal 
with overwhelming probability, e.g., by Hoeffding's inequality.

Next, we shall impose an incoherence condition on $\mu^{(i)}$, namely there is a function\footnote{Note that the signal amplitude $s(n) = O(1)$ always. 
In some cases, we may let $s(n)$ scale down as $n$ increases, e.g., $1/\sqrt{n}$, which corresponds to a `hard' setup for inference.} $s(n)$ 
such that $\|\mu^{(i)} \|_\infty = \Theta \big( s(n) \big)$ and $\|\mu^{(i)} \|_2 = \Theta \big( n s(n) \big)$ for all $i \in [k]$. Intuitively, this condition 
implies that information (signal) about the common preference in each RUM component is well spread over pairwise comparisons and is 
easily satisfied under mild conditions. With this incoherence assumption, $ \| M \|_F = \Theta \big( s(n) \sqrt{Nn^2} \big)$.

%assume that where 
%$s(n) = O(1)$\footnote{Note that $\mu^{(i)}_{(a,b)} \in [-\frac{1}{2}, \frac{1}{2}]$ for all $1 \leq a < b \leq n$ and for all $i \in [k]$. In some cases, 
%$s(n)$ could be a quantity that scales down as $n$ increases, e.g., $1/\sqrt{n}$.}. %, e.g. $s(n)$ could be constant or quantity scaling down like $1/\sqrt{n}$.
%As we shall explain through various examples in   Section \ref{sec:examples}, this is a natural assumption in our setup. Consequently, the gap $\min_{a\neq b \in [r]} \|\mu^a - \mu^b\|_2$ will scale with $s(n)$. 
%To that end, the gap scale as $\Gamma_n = s(n) \Gamma$ where $\Gamma$ is the scale $s(n)$, independent
%aspect of the model and $\Gamma = O(n)$. 

When $k \ll n^2$, we expect $r = \rank(M) \approx k$ because there are $k$ different types of rows that are nearly orthogonal to each other. 
Moreover, $\mu^{(i)}$ are closely related to the right singular vectors of $M$ and the frequency of each type in $M$ is related to $\sigma_i(M)$. 
Specifically, letting $\hat{p}_i = \frac{1}{N} \sum_{\ell =1}^N \Ind{ \alpha^*(\ell) = i }$ denote the mixture weights ($\approx p_i$) for $i \in [k]$, 
we expect that
\begin{align}
%\sigma_1(M)\approx \dots \approx & \approx \sqrt{\frac{1}{r} \| M \|_F^2 }	\nonumber\\
	\sigma_r(M) &= \Theta\bigg( s(n) \sqrt{\big(\min_{i \in [k]} \hat{p}_i \big) Nn^2}\bigg).				\label{eq:singular.values}
\end{align} 

%On the other hand, when $k \gg n^2$, we expect $k \gg n^2 \geq r$. By the near-orthogonality assumption, 
%$\sigma_1(M)\approx \dots \approx \sigma_r(M) = \Theta\Big( s(n) \sqrt{ Nn^2 / r}\Big).$

We shall assume $\tau^* \geq 1$ %(by re-defining $\tau^\star = \max(1, \tau^\star)$) 
to avoid trivialities where $\tau^\star$ is too small. %(in which case, mixture learning is almost trivial).
%, e.g. the mixture of a few Dirac deltas that are well separated). 
Under the near-orthogonality and incoherence assumptions described above, we obtain the following Corollary of Theorem \ref{thm:main}. 

\begin{corollary}\label{cor:main}
Suppose that $n$ is sufficiently large and $k \ll n^2$. If $\hat{p}_i = \Theta(1/k)$ for all $i \in [k]$; $s(n) = \omega(\sqrt{r/n})$; $N = n^{4}$; and 
\begin{align}\label{eq:cor.main} 
	p & \geq \frac{C'' \tau^\star \sqrt{r} \log n}{\Gamma}
\end{align}
for a sufficiently large constant $C'' > 0$, then with probability at least $1 - O(1/n^2)$, Algorithm \ref{alg:main_algorithm} with proper $t_1, t_2$ 
successfully clusters $N$ rankings.
\end{corollary}

\begin{remark}\label{rem:number_comparisons}
Given $n, p, N$, the expected number of pairwise comparisons in the data is ${n \choose 2} N p$. Corollary \ref{cor:main} implies that 
having $\Omega \big(\frac{ \sqrt{r} \tau^\star n^{6} \log n }{\Gamma} \big)$ number of (pairwise) comparisons is sufficient for 
successful clustering. 
\end{remark}

\begin{remark}\label{rem:samples_necessary}
We let $N = n^{4}$ in Corollary \ref{cor:main} for simplicity and this choice is not necessary. It is easy to see from \eqref{eq:singular.values} 
that $p \sigma_r(M) = \Omega \big( \tau^* n^2 / \sqrt{r}\big)$ if $N \geq n^2 {\tau^\star}^2 / p^2$. The proof of Corollary \ref{cor:main} 
remains valid as long as $N \geq n^2 {\tau^\star}^2 / p^2$. Therefore, we can conclude that having ${n \choose 2} Np = \frac{n^4 {\tau^\star}^2}{p}$ 
number of pairwise comparisons is sufficient for clustering, which can be as small as $\Omega ( n^4 )$ under a nice setup where 
$\tau^* = O(1)$ and $p = \Theta(1)$.
\end{remark}

\paragraph{Generic Upper Bound on $\tau^*$.}
As a first step to enable concrete evaluation of Theorem \ref{thm:main}, we establish a generic bound on sub-Gaussian norm for 
{\em any} instance of RUM, which may be of interest in its own right.
\begin{proposition}\label{prop:subg_rum}
For any given instance of Random Utility Model, there exists universal constant $C > 0$ so that associated sub-Gaussian
parameter $\tau \leq C \sqrt{n}$. Therefore, $\tau^\star \leq C \sqrt{n}$. 
\end{proposition}
Our proof of Proposition \ref{prop:subg_rum} is based on bounded-difference concentration inequality.
Some empirical discussions on the validity of this estimate can be found in Section \ref{sec:experiments}. 
In short, the upper bound seems sharp in the low signal-to-noise ratio (SNR) regime and quite loose in the 
high SNR regime for both MNL and Gaussian RUM.

%Using Proposition \ref{prop:subg_rum} with Theorem \ref{thm:main}, we are ready to state a generic result that establishes
%polynomial dependence of number of samples required on various model characteristics in order to learn the mixtures. To that end, 
%we shall consider the following setting.

%%% ======================================================================
%%%	Examples: MNL Model and RUM with Gaussian Noise
%%% ======================================================================

\section{EXAMPLES}\label{sec:examples}

In this section, we consider two concrete examples -- mixture MNL model and mixture Gaussian RUM. Our analysis 
in these examples relates mixture parameters such as $\Gamma$ to the model parameters of each RUM, 
e.g., SNR, thereby hinting sufficient conditions for successful clustering. We assume the deterministic utilities $u^{(i)}$ % = (u_1^{(i)}, \ldots, u_n^{(i)})$ 
for components $i \in [k]$ are randomly drawn.

%%% ======================================================================
%%%	MNL
%%% ======================================================================

\subsection{Multinomial Logit (MNL) Model}\label{sec:example_mnl}

\paragraph{Description of the Setup.}
We consider the MNL model with scaling parameter $\beta > 0$, cf. Example \ref{example:mnl}. 
Recall from \eqref{eqn:prob_mnl} that $\Prob{Z_a \geq Z_b } = \frac{w_a}{w_a + w_b}$ for any pair $a, b \in [n]$ where
$w_a \triangleq \exp \big( \frac{u_a}{\beta} \big)$ denotes the `weight' for $a \in [n]$.  We impose the following two assumptions:
%\begin{enumerate}
%	\item
(1) $u_1 > u_2 > \ldots, u_n$ and there exists $\rho > 0$ such that $u_a - u_{a+1} \geq \rho$ for all $a \in [n-1]$; 
(2) 
%	\item
for each $i \in [k]$, $u_a^{(i)} = u_{\phi_i(a)}$ for some $\phi_i \in S_n$ that is drawn i.i.d. from the discrete uniform distribution over $S_n$. 
%\end{enumerate}

\paragraph{Separation between Cluster Centers.} 
Assumption 1 implies $\frac{w_a}{w_{a+1}} \geq \exp( \frac{\rho}{\beta})$ for all $a \in [n-1]$. That is, the true deterministic utilities 
of the $n$ options are well separated from each other. It follows that  
\begin{align*}
	\Prob{Z_a \geq Z_b } &= \frac{w_a}{w_a + w_b} \geq  \frac{1}{ 1 - \exp \big( - \frac{\rho}{\beta} \big)},
		\qquad\text{and}%\\
\end{align*}
\begin{align*}
%		
%		\qquad
	\Prob{Z_a < Z_b } &= \frac{w_b}{w_a + w_b} \leq  \frac{\exp \big( - \frac{\rho}{\beta} \big)}{ 1 - \exp \big( - \frac{\rho}{\beta} \big)}
\end{align*}
for any $a, b \in [n]$ with $a < b$. %We remark here that $ \bbE X_{(i,j)} = \frac{1}{2} \Prob{Z_i \geq Z_j} - \frac{1}{2} \Prob{Z_i < Z_j}  
%\geq \frac{1}{2} \frac{1-\rho}{1+\rho}$ for all $1 \leq i < j \leq n$.
Based on this observation, we obtain a high-proability lower bound on $\Gamma$.

\begin{proposition}\label{prop:mixture_mnl}
With probability at least $1 - \frac{ r^2}{n^4}$,
\[	\Gamma \geq \frac{\sqrt{n(n-1)}}{2}  \frac{ 1 - \exp \big( - \frac{\rho}{\beta} \big)}{ 1 + \exp \big( - \frac{\rho}{\beta} \big)} - 4 \sqrt{n \log n}.	\]
\end{proposition}

%In short,
%\begin{align*}
%	\Gamma \geq C n \frac{1-\rho}{1+\rho}	\qquad\text{and}\qquad \tau \leq C' \sqrt{n}.
%\end{align*}

\paragraph{Implications.}
We interpret Corollary \ref{cor:main} in the mixture MNL setup. Suppose that $n$ is sufficiently large. 
To begin with, we observe that $s(n) \geq \frac{1}{2} \frac{ 1 - \exp ( - \rho / \beta )}{ 1 + \exp ( - \rho / \beta )}
= \omega( \sqrt{r/n} )$ if $\rho \gg - \beta \ln \big( 1 - \sqrt{r/n} \big) \approx \beta \sqrt{r/n}$.
It follows from Proposition \ref{prop:mixture_mnl} that $\Gamma \geq \frac{n}{4} \frac{ 1 - \exp ( - \rho / \beta )}{ 1 + \exp ( - \rho / \beta )}$ 
with high probability. We know that $\tau^{\star} \leq C \sqrt{n}$ for some $C > 0$ by Proposition \ref{prop:subg_rum} and hence
\eqref{eq:cor.main} is satisfied if 
\begin{equation}\label{eq:ex1.p.simplified}
	p \geq C'''  \frac{\beta}{\rho} \frac{\sqrt{r} \log n}{\sqrt{n}}.
\end{equation}
We approximated $ \frac{ 1 + \exp ( - \rho / \beta )}{ 1 - \exp ( - \rho / \beta ) } \approx \frac{\beta}{\rho}$ assuming $\rho \gg \beta$.
Note that the scaling parameter $\beta$ denotes the typical magnitude %\footnote{The variance of the Gumbel distribution is $\frac{\pi^2 \beta^2}{6}$.} 
of the Gumbel noise and therefore $\frac{\beta}{\rho}$ is the inverse signal-to-noise ratio in the MNL model. 

Therefore, we can conclude that as long as $\rho \gg \beta \sqrt{r/n}$, having 
$\Theta \big( \frac{\beta}{\rho} r^{0.5} n^{5.5} \log n \big)$ number of pairwise comparisons 
is sufficient for successful clustering; see Corollary \ref{cor:main} and Remark \ref{rem:number_comparisons}.

%\paragraph{Refined Analysis of Sample Complexity.} 
%The concrete form of mixture MNL model enables a more refined analysis for the sample complexity by directly interpreting Theorem \ref{thm:main} 
%rather than Corollary \ref{cor:main}. Recall from \eqref{eq:gap.condition} that we need $\Gamma \geq C' \frac{\tau^\star \sqrt{r} \log n}{p}$ for 
%a sufficiently large constant $C' > 0$ to apply Theorem \ref{thm:main}. Observe that when $\exp\big( -\frac{\rho}{\beta} \big) \leq 1 - C''' \frac{\tau^{\star} \sqrt{r} \log n}{p n}$, this condition is satisfied with high probability according to Proposition \ref{prop:mixture_mnl}. This sufficient condition 
%can be rewritten as $\rho \geq C'''' \beta \frac{\tau^{\star} \log n}{n p}$ or equivalently, 
%\begin{equation}\label{eq:ex1.snr}
%	p \geq C''''' \frac{\beta}{\rho}\frac{\tau^{\star} \sqrt{r}\log n}{n}.	
%\end{equation}
%Note that the sample complexity lower bound in \eqref{eq:ex1.snr} is much sharper than that in \eqref{eq:ex1.p.simplified}; specifically
%resulting in requiring $\Theta \big( \frac{\beta}{\rho} r^{0.5} n^{5} \log n \big)$ pairwise comparisons. 

%%% ======================================================================
%%%	Gaussian noise
%%% ======================================================================

\subsection{RUM with Gaussian Noise}\label{sec:example_gaussian}

\paragraph{Description of the Setup.}
Now we consider the random utility model with Gaussian noise with $\sigma_a = \sigma > 0$ for all $a \in [n]$; 
cf. Example \ref{example:gaussian}. That is, for $1 \leq a < b \leq n$ and for $\ell_1, \ell_2 \in [N]$,
\begin{align*}
	X_{(a,b)}^{\ell_1} - X_{(a,b)}^{\ell_2} %= \Ind{ Z_i^a \geq Z_j^a } - \Ind{ Z_i^b \geq Z_j^b }
		&= \Ind{ \eps_a^{(i)} - \eps_b^{(i)} \geq u_b^{(i)} - u_a^{(i)} } \\
		&\quad	- \Ind{ \eps_a^{(j)} - \eps_b^{(j)} \geq u_b^{(j)} - u_a^{(j)} }
\end{align*}
where $i = \alpha^*(\ell_1)$, $j = \alpha^*(\ell_2)$ are shorthands. Observe that $\eps_a^{(i)} - \eps_b^{(i)}, \eps_a^{(j)} - \eps_b^{(j)}  
\sim \cN(0, 2\sigma^2 )$, i.e., the normal distribution with mean $0$ and variance $2\sigma^2$. 

Suppose that for each $i \in [k]$, the deterministic utility $u^{(i)} \in \{ \pm \frac{1}{2} \}^n$ is chosen from the hypercube uniformly at random. 
Precisely, $u_a^{(i)}$ are i.i.d. random variables taking value$\frac{1}{2}$ or $-\frac{1}{2}$ with probability 
$\frac{1}{2}$, respectively.
%Since $\sigma_i = \sigma > 0$ for all $i \in [n]$, 

\paragraph{Separation between Cluster Centers.}
For $x \in \Reals$, let $\Phi(x) = \int_{-\infty}^{x} \frac{1}{\sqrt{2 \pi}} e^{-\frac{t^2}{2}} dt$ denote the cumulative distribution 
function of the standard normal distribution. Under the setup described above, we obtain a high-proability lower bound on $\Gamma$ as follows.

\begin{proposition}\label{prop:mixture_gaussian}
With probability at least $1 - \frac{r^2}{n^4}$,
\begin{equation}\label{eqn:gamma_lower_gaussian}
	\Gamma \geq  \frac{\sqrt{n(n-1)}}{\sqrt{2}}\bigg( \Phi \Big( \frac{1}{\sigma \sqrt{2}} \Big) - \frac{1}{2} \bigg)
		 - 4 \sqrt{n \log n}.
\end{equation}
\end{proposition}

%\begin{remark}\label{rem:mixture_gaussian}
%We have
%$\Phi \Big( \frac{1}{\sigma \sqrt{2}} \Big) - \frac{1}{2} \geq 
%\max \big\{ \frac{1}{ \sigma} \exp \left( - \frac{1}{4\sigma^2} \right),$ $\frac{1}{2} - \frac{ \sigma}{\sqrt{\pi}} \exp \left( - \frac{1}{4 \sigma^2} \right) \big\}$. 
%When $\sigma$ is large, the first term in $\max\{ \cdot \}$ gives a sharper lower bound, whereas the second term 
%provides a tighter lower bound when $\sigma$ is small.
%\end{remark}

\paragraph{Implications.}
Suppose that $n$ is sufficiently large and $\sigma$ is a constant that does not scale with $n$. 
%First of all, we apply Corollary \ref{cor:main} for this setup. Observe that $s(n) = \omega\big( \sqrt{r/n} \big)$ when $r = o(n)$. 
Assuming $\sigma = \Theta(1)$, $\Gamma \geq C \frac{n}{\sigma}$ with high probability by Proposition \ref{prop:mixture_gaussian}. 
%and Remark \ref{rem:mixture_gaussian}. 
Since $\tau^{\star} \leq C \sqrt{n}$ for some $C > 0$ by Proposition \ref{prop:subg_rum}, 
\eqref{eq:cor.main} is satisfied if 
\begin{equation*}\label{eq:ex1.p.simplified_gauss}
	p \geq C' \sigma \frac{\sqrt{r} \log n}{\sqrt{n}}.
\end{equation*}
This implies having $\Theta \big( \sigma r^{0.5} n^{5.5} \log n \big)$ number of pairwise comparisons is sufficient for successful 
clustering of rankings from a mixture of Gaussian RUMs; cf. Corollary \ref{cor:main} and Remark \ref{rem:number_comparisons}.
%We note here that there is a chance a tighter sample complexity lower bound like \eqref{eq:ex1.snr} is attainable.

%Again, we provide a more refined analysis on the sample complexity utilizing the concrete model. Recall from \eqref{eq:gap.condition} 
%that we need $\Gamma \geq C'' \frac{\tau^{\star} \sqrt{r} \log n}{p}$ for a sufficiently large constant $C'' > 0$ to apply Theorem \ref{thm:main}. 
%As in Section \ref{sec:example_mnl}, we want to find an upper bound on $\sigma$ (or equivalently, the lower bound on the SNR) and 
%the lower bound on $p$. 
%Note that this condition is satisfied when $\sigma = O \Big(\frac{n p}{ \tau^\star \sqrt{r} \log n} \Big)$.
%%This condition is satisfied when $  \frac{1}{\sigma} \exp \big( - \frac{1}{4\sigma^2} \big) \geq \frac{r \log^2 n}{n p^2}$. 
%
%Thus, our analysis shows that to apply Theorem \ref{thm:main}, the fraction of observations, $p$, must satisfy
%\begin{align}\label{eq:ex2.snr}
%	p \geq C''' \frac{\sigma \tau^\star \sqrt{r} \log n}{n}.	
%\end{align}

%%% ======================================================================
%%%	Experiments
%%% ======================================================================

\section{EXPERIMENTS}\label{sec:experiments}
We run experiments on synthetic datasets to show the efficacy of the proposed clustering algorithm. 

\noindent{\em Parameters.} The datasets are generated using the following parameters: 
$n$ is the number of options; $k$ is the number of RUM components in the mixture; 
$\lambda$ is the expected number of rankings per component; and 
$\beta / \sigma$ represents the noise level.
%\begin{align*}
%	&\bullet~ n: &&\text{number of options}\\
%	&\bullet~ k: &&\text{number of RUM components in the mixture}\\
%	&\bullet~ \lambda: &&\text{expected number of rankings per component}\\
%	&\bullet~ \beta / \sigma: &&\text{noise level}
%\end{align*}

\noindent{\em Data generation process.}
We generate synthetic datasets as follows: 
(1) For $i \in [k]$, draw $u^{(i)} \sim \cN(0, I_n)$ and $N_i \sim \text{Poisson}(\lambda)$, i.i.d.; 
and (2) for $i \in [k]$, instantiate $Z_1, \ldots, Z_n $ according to MNL model (or Gaussian RUM) with utility $u^{(i)}$ 
and noise parameter $\beta$ (or $\sigma$) i.i.d. $N_i$ times.   

%\subsection{The Proposed Algorithm Works}\label{sec:exp_works}

%\paragraph{Experiment 1: Algorithm \ref{alg:main_algorithm} Works.}
\noindent{\bf Experiment 1: Algorithm \ref{alg:main_algorithm} Works.}
First of all, we exhibit how the proposed clustering algorithm works at different noise levels. We generate synthetic datasets according 
to the Gaussian noise model with parameters $n = 30$, $k=100$, $\lambda = 50$ at three different noise levels $\sigma \in \{ 0.3, 0.5, 1.0\}$. 

\begin{figure}[h]
	\centering
	\includegraphics[width=0.75\linewidth]{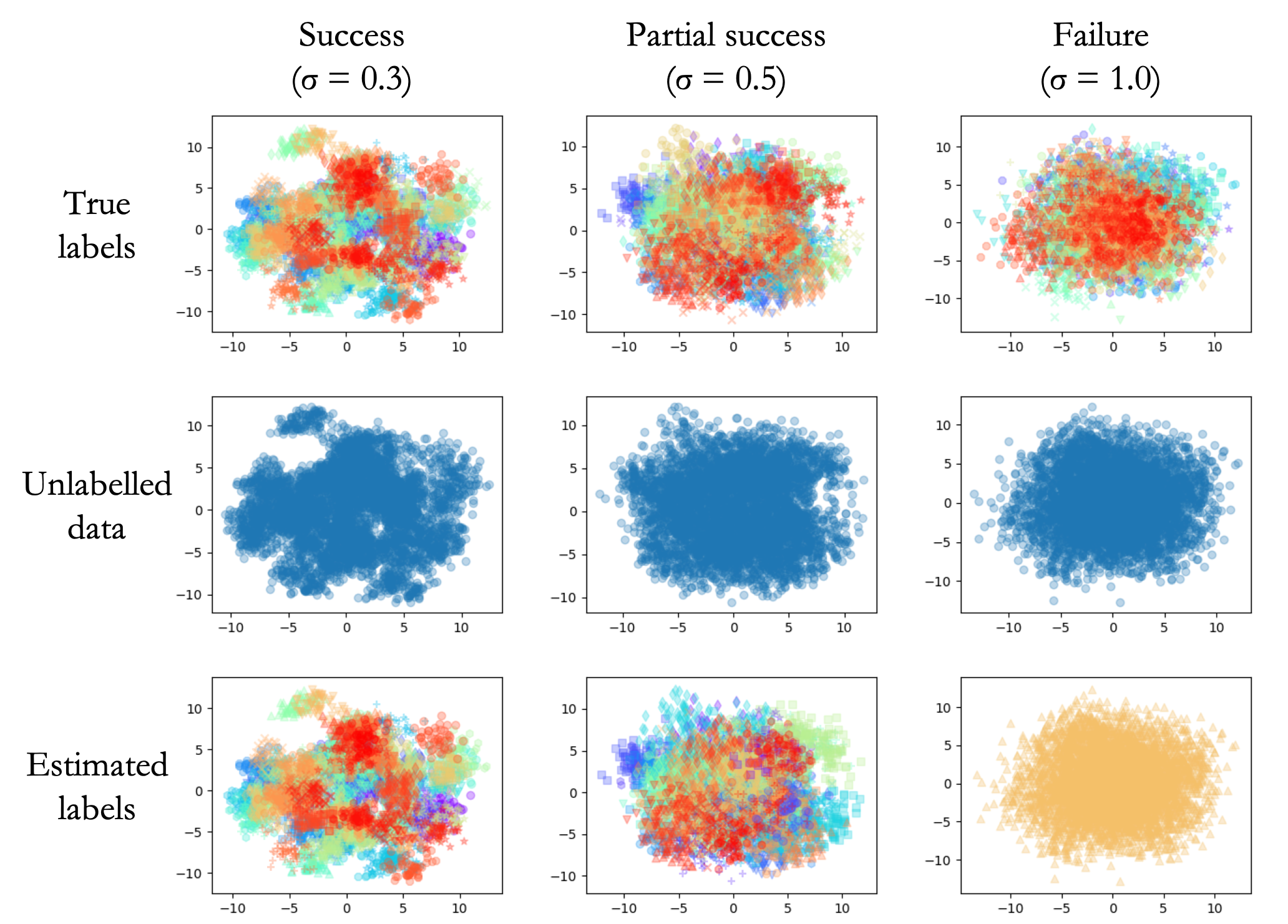}
	\caption{Illustration of how Algorithm \ref{alg:main_algorithm} works for Gaussian noise model at three noise levels. }
	\label{fig:algorithm_works}
\end{figure}

\begin{figure}[h]
	\centering
	\includegraphics[width=0.75\linewidth]{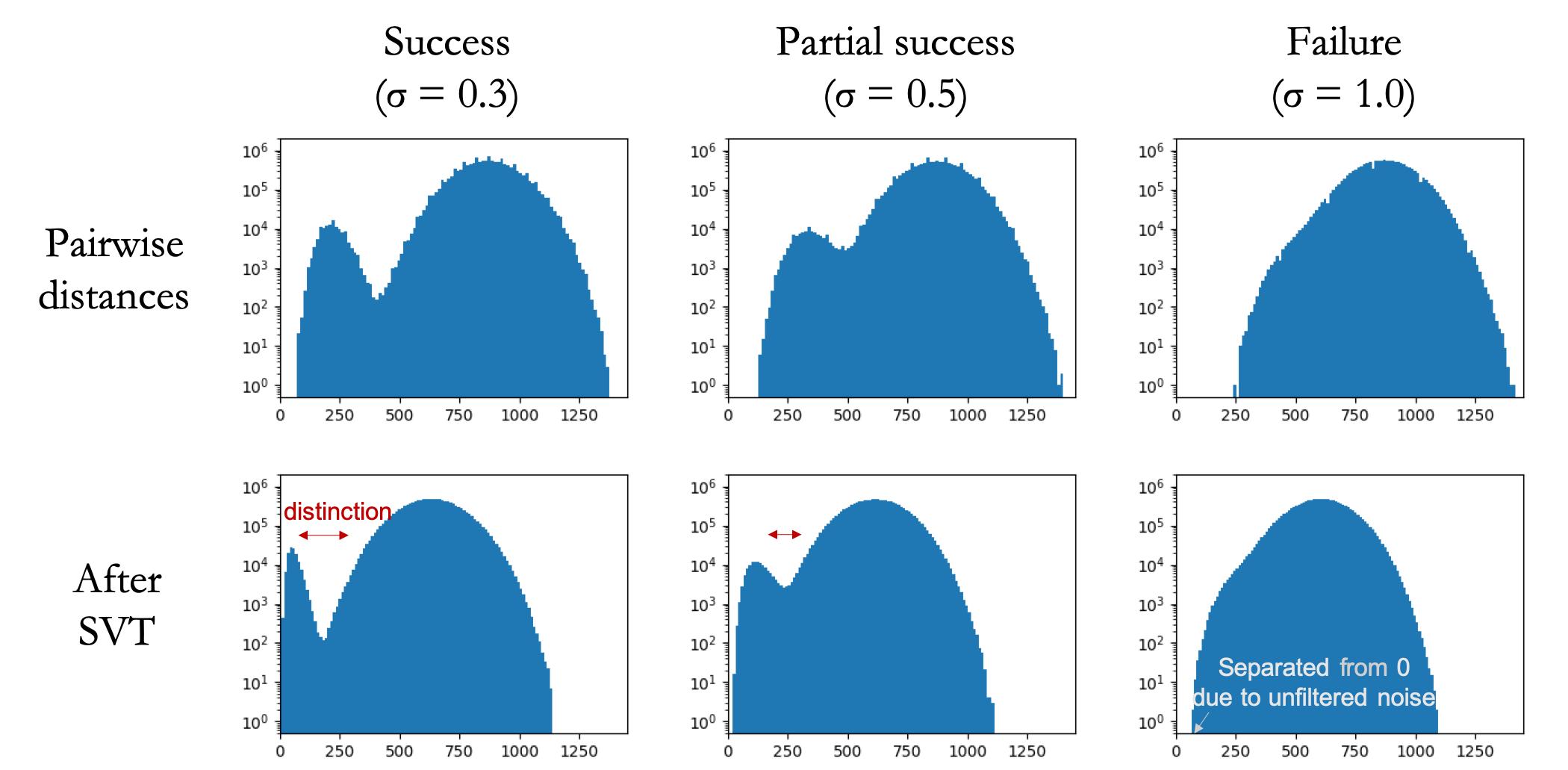}
	\caption{Distributions of pairwise distances in three scenarios: before SVT (top) and after SVT (bottom). }
	\label{fig:pairwise_dist}
\end{figure}

Figure \ref{fig:algorithm_works} illustrates the results for three possible scenarios, namely, success ($\sigma = 0.3$), 
partial success ($\sigma=0.5$), and failure ($\sigma = 1$). The first row in Figure \ref{fig:algorithm_works} displays 
the ranking profiles with color-coded labels (=index of the cluster) in the format of ${30 \choose 2}$-dimensional pairwise 
marginals projected onto the 2D plane of top-$2$ principal components. Our algorithm sees unlabelled datasets 
as depicted in the second row and outputs estimated labels as in the third row. 

Recall that Algorithm \ref{alg:main_algorithm} classifies rankings based on pairwise distance. Figure \ref{fig:pairwise_dist} 
displays the histogram of pairwise $\ell_2$ distances (squared) between the ranking profiles for each cases in the semi-log scale.
We observe that SVT shaves off the noise contribution in the pairwise distance, thereby shifting the peaks closer to $0$. 
In the success ($\sigma=0.3$) case, there is a significant separation between the distribution of intra-cluster distances (left peak) 
and the distribution of inter-cluster distances (right peak) so that the algorithm easily reconstructs the clusters. On the other hand, 
there is no distinction between the peaks in the failure ($\sigma=1.0$) case. Moreover, we expect the left end of the distance 
distribution to touch $0$ when $N$ is sufficiently large. The gap between the left end and $0$ in the bottom right plot suggests 
the contribution of residual noise, possibly due to imperfect identification of principal components caused by insufficient SNR.

%\paragraph{Experiment 2: Overcoming the Missing Values.}
\noindent{\bf Experiment 2: Overcoming the Missing Values.} 
Next, we show Algorithm \ref{alg:main_algorithm} can cluster the ranking profiles even when some pairwise comparisons are absent. 
We consider Gaussian RUM mixture instances with $n = 30$, $k=2$ and $\lambda = 500$. Our algorithm has access to 
$p \in [0,1]$ fraction of the data of pairwise marginals, cf. see \eqref{eq:YnX}. We measure the performance of the algorithm 
using the misclassification rate 
$\Risk(\halpha) = \min_{\sigma \in S_{2}} \frac{1}{N} \sum_{\ell=1}^N \Ind{ \halpha(\ell) \neq \alpha^*(\sigma(\ell)) }$. This
is plotted in Figure \ref{fig:tau_estimates} as we vary $p$. 
Ideally, we expect $\Risk(\halpha) = 0$ if the algorithm succeeds in perfect clustering and $\Risk(\halpha) \approx 0.5$ 
when the algorithm completely fails in extracting any information about the underlying RUM components. As seen
from Figure \ref{fig:tau_estimates}, the misclassification rate $\to 0.5$ as $p \to 0$, and the rate goes to $0$ as $p$
increases with exception of the setting with noise is very high $\sigma = 1$ or $n$ is small, i.e. $n = 20$. 

\begin{figure}[h]
\centering
	\begin{subfigure}{0.45\linewidth}
		\centering
		\includegraphics[width=\linewidth]{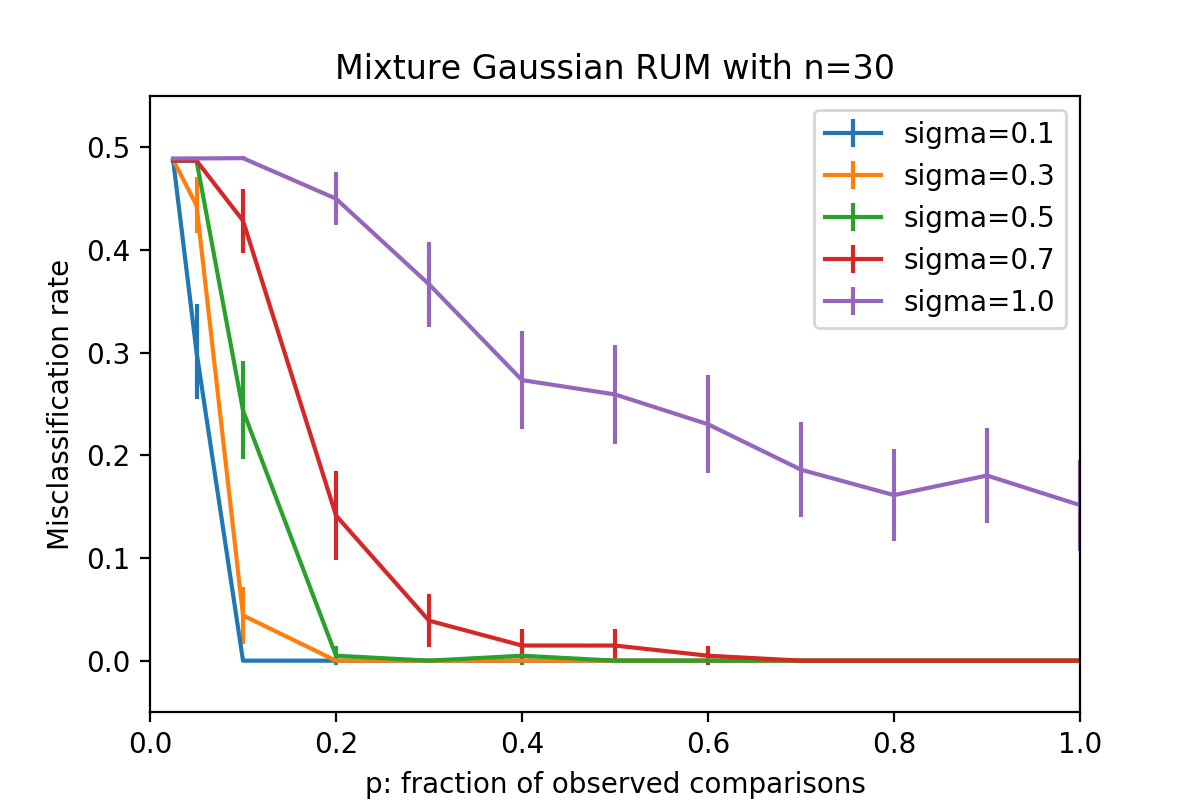}
%		\caption{Misclassification rate vs $p$ at various $\sigma$}
%		\label{fig2:mnl}
	\end{subfigure}
	\hfill
	\begin{subfigure}{0.45\linewidth}
		\centering
		\includegraphics[width=\linewidth]{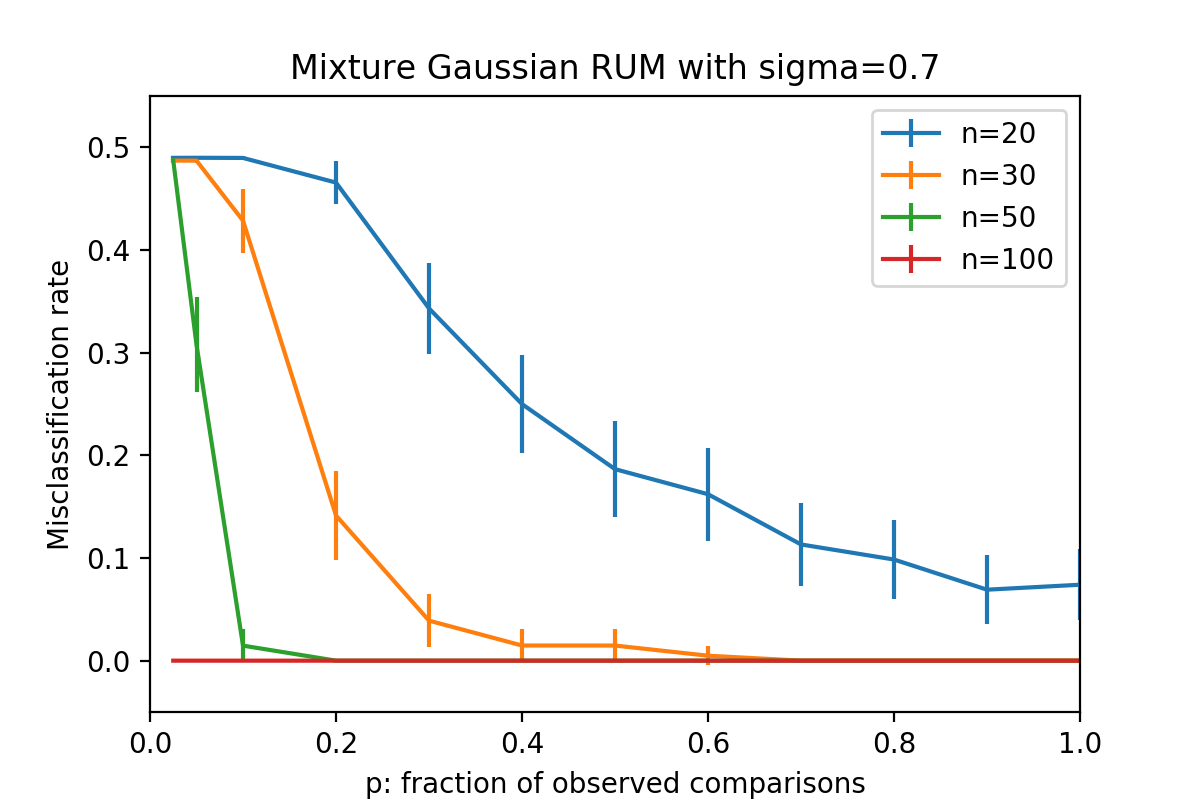}
%		\caption{Misclassification rate vs $p$ for various $n$}
%		\label{fig2:gaussian}
	\end{subfigure}
	\caption{Error rate of Algorithm \ref{alg:main_algorithm} on partially observed data from Gaussian RUM mixtures: 
		various $\sigma$ with $n=30$ (left); various $n$ with $\sigma=0.7$ (right).}
	\label{fig:tau_estimates}
\end{figure}

%\paragraph{Experiment 3: Estimates for $\tau^*$.}
\noindent{\bf Experiment 3: Estimates for $\tau^*$.}
Recall that our results (Theorem \ref{thm:main} and Corollary \ref{cor:main}) depend on the sub-Gaussian parameter $\tau$ (see Definition \ref{def:subG}). 
It is known that the sub-Gaussian parameter is equivalent to the Orlicz $\psi_2$ norm\footnote{see Definition \ref{defn:psi2} in Appendix 
\ref{sec:preliminaries_subG}}, i.e., $\tau = c \| W - \bbE W \|_{\psi_2}$ for an absolute constant $c > 0$. We estimate the $\psi_2$ norm 
by replacing the expectation in Definition \ref{defn:psi2} with the empirical mean from 1000 randomly generated samples.

This procedure is repeated $100$ times for each pair of $\beta$ (or $\sigma$) ranging from $0.05$ to $1.0$ and $n$ from $2$ to $200$. 
The results are summarized in Figure \ref{fig:tau_estimates}. For both MNL and Gaussian RUMs, the slope tends to $0.5$ as $n$ grows. 
This trend suggests that the growth rate of $\sqrt{n}$ is correct when the noise level is high compared to the difference between utilities, 
i.e., when $|\eps_a| \approx \min | u_a - u_b|$. However, the rapid drop on the left end also suggests $\tau^*$ can be significantly smaller than 
the generic upper bound when $|\eps_a| \ll \min | u_a - u_b|$.

\begin{figure}[h]
\centering
	\begin{subfigure}{0.45\linewidth}
		\centering
		\includegraphics[width=\linewidth]{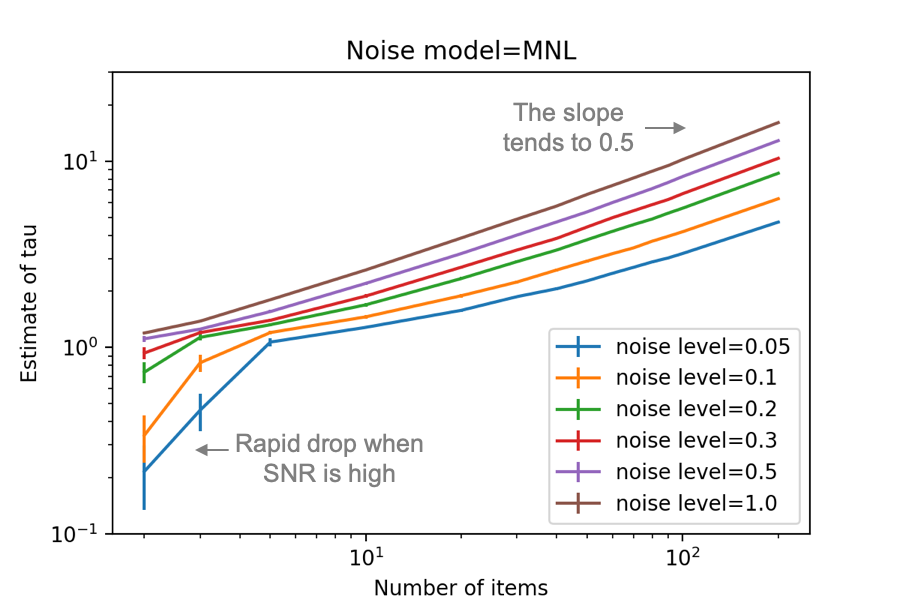}
%		\caption{Estimates of $\tau$ for MNL model}
%		\label{fig2:mnl}
	\end{subfigure}
	\hfill
	\begin{subfigure}{0.45\linewidth}
		\centering
		\includegraphics[width=\linewidth]{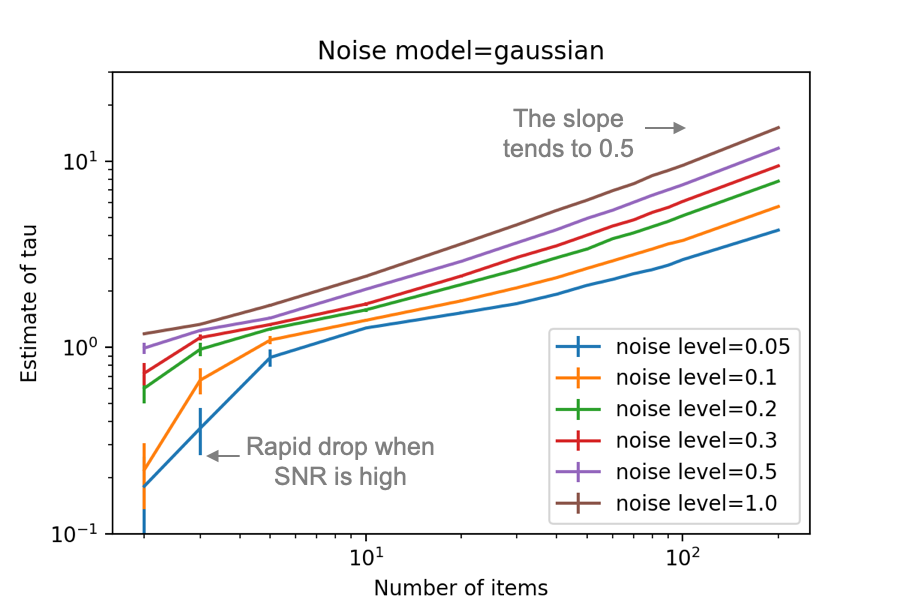}
%		\caption{Estimates of $\tau$ for Gaussian RUM}
%		\label{fig2:gaussian}
	\end{subfigure}
	\caption{Log-log plots of the estimated value of $\tau$ for MNL models (left) and Gaussian RUMs (right).}
	\label{fig:tau_estimates}
\end{figure}

%%% ======================================================================
%%%	Discussion
%%% ======================================================================

\section{DISCUSSION}\label{sec:discussion}
We proposed a procedure to reduce the challenging task of learning mixture RUMs to simpler problems of single RUM learning. 
The procedure is inspired by spectral clustering approach for mixture learning in vector space and we provide theoretical guarantees 
on the statistical and computational performance of an explicit instance described in Algorithm \ref{alg:main_algorithm}. For example, 
we identify a sufficient condition for successful clustering, cf. Theorem \ref{thm:main} and \eqref{eq:gap.condition}, which suggests 
the tradeoff between intrinsic separability of the model ($\Gamma$) and the required fraction of data ($p$). We conclude with a few remarks on possible next steps.

\paragraph{Lower Bound on $p$.}
Our results, e.g., \eqref{eq:ex1.p.simplified}, suggests $p = \tilde{\Omega}( \sqrt{r/n})$ is sufficient for clustering mixture RUMs. 
However, we do not know whether this is necessary, or $p$ can scale as small as the order of $r/n$. Confirming a lower bound can 
be an interesting direction for future work. 

\paragraph{Use of Other Matrix Estimation Methods.}
We deploy PCA (SVT) in this work to recover the hidden structure in the mean matrix $M$, but it is be possible to utilize 
other types of matrix estimation methods. Obtaining provable guarantees for such methods, e.g., low-rank matrix estimation 
with nuclear norm penalization or collaborative filtering, could be interesting. 

\paragraph{Beyond the Requirement of $\tilde{\Omega}(\sqrt{r})$ Separation.}
Note that $p$ cannot exceed $1$ and therefore Theorem \ref{thm:main} is meaningful only when $\Gamma = \tilde{\Omega}(\tau^\star \sqrt{r})$, 
cf. \eqref{eq:gap.condition}. This condition is consistent with the well-known $\sqrt{r}$ barrier in second-order methods for spectral clustering. 
In fact, if the objective is in parameter learning rather than clustering, one may break this barrier via ideas from robust statistics, e.g., SoS certificates. 
It would be interesting to adapt the techniques to learning a heterogeneous mixture of generic RUMs, which we consider to be a semiparametric model.

\bibliography{bibliography}
\newpage

\onecolumn
\appendix

%%% ======================================================================
%%%	Notation
%%% ======================================================================
\section{Notation}

\medskip
\noindent{\em Generic.} For a positive integer $n$, we let $[n] \triangleq \{ 1, \ldots, n \}$. Given $a, b \in \Reals$, we let 
$a \vee b = \max\{a, b\}$ and $a \wedge b = \min\{a, b \}$. $\| \cdot\|_{2}$ denotes 
the $\ell_2$ norm for a vector, and the spectral norm for a matrix. For a random variable and 
a random vector, $\|\cdot\|_{\psi_2}$ denotes the Orlicz $\psi_2$ norm. We let $C$ denote some constant 
and the exact value of $C$ can change from line to line.

\medskip
\noindent {\em Notational change for simplicity.}  For ease of notation, through out the technical sections, 
we shall use $d = {n \choose 2}$ to represent the 
dimensionality of permutation of $n$ items embedded in pair-wise comparison representation. That is, we may 
represent $Y$ (and $X$, $M$) as $N \times d$ dimensional matrix for simplicity rather than $N \times {n \choose 2}$
dimensional. Also, for notational convenience, we shall denote the $Y_{ij}$ for $i \in [N], j \in [d]$ to represent
the pair-wise ordering of a pair of items in $[n]$, indexed as $j$ amongst $d = {n \choose 2}$ possible different
pairs, in the $i$th permutation sampled as described in Section \ref{sec:model}. This is in contrast to
indexing elements of permutation by tuple $(a, b)$ for representing pair-wise comparison between $a, b \in [n]$. 

\medskip
\noindent {\em And some.}  For matrix  $X \in \Reals^{N \times d}$, we write the Singular Value Decomposition (SVD) 
of $X$ as $X = U \Sigma V^T$, where  $\Sigma = diag\big( \sigma_1(X), \ldots, \sigma_{\min\{N, d\}}(X) \big)$ 
with the singular values of $X$, $\sigma_1(X) \geq \sigma_2(X) \geq \ldots \geq \sigma_{\min\{N, d\}}(X)$ in 
the descending order.

For each $i \in [N]$, we let $X^i = X_{i*}$ denote the $i$-th row of the matrix $X$, viewed as a row-vector in 
$\Reals^{1 \times d}$. For each $j \in [d]$, we let $X_j = X_{*j}$ denote the 
$j$-th column of the matrix $X$, viewed as a column-vector in $\Reals^{N \times 1}$.
When $X$ is viewed as a random matrix, we reserve $M$ to denote its expectation, i.e., $M = \Exp{X}$. 
Then, $M^i$ is the $i$th row of $M$. We reserve $r$ to denote the rank of $M$.

%\input{contents/Appx02_key_lemma_and_proof}
%%% ======================================================================
%%%	Effect of Missing Values
%%% ======================================================================

\section{Technical Result: Effect of Missing Values}\label{sec:missing_values}

Recall definition of matrix $Y$ which is obtained by stacking samples from mixture of RUM
with missing values as described in Section \ref{sec:model}. The $X$ is random matrix where 
all entries are observed, i.e. $Y$ is randomly (Bernoulli($p$)) masked version of $X$, cf. \eqref{eq:YnX}. Recall
that matrix $M = \Ex{X}$. By definition, $\Ex{Y} = p \Ex{X} = pM$. Now $Y = pM + \big( Y - pM \big)$. 
We show that the perturbation matrix matrix $Y - pM$ has independent, mean zero, sub-gaussian rows.

%%% ======================================================================
%%%	Basic Statistics of  $Y^i - pM^i$
%%% ======================================================================

\subsection{Mean and Covariance of $Y - pM$}
The following two lemmas are consequences of the i.i.d. Bernoulli mask in our model.
\begin{lemma}\label{lem:mean_Y}
When $\star$ is replaced with $0$,
$\Exp{Y - pM} = 0.$
\end{lemma}
\begin{proof}
	Recall that $Y_{ij} = X_{ij} \Ind{E_{ij}=1}$ for each $(i,j) \in [N] \times [d]$. 
	Now observe that for all $(i,j) \in [N] \times [d]$,
	\begin{align*}
		\Exp{ Y_{ij} - pM_{ij}} 
			&= \Exp{ X_{ij} \Ind{E_{ij} = 1} - pM_{ij} }\\
			&= \Exp{X_{ij} } \Exp{\Ind{E_{ij} = 1} } - p M_{ij}\\
			&= 0.
	\end{align*}
	We used the independence between $X_{ij}$ and $E_{ij}$ % (see generative model; Figure \ref{fig:data_generation}) 
	as well as the fact that $ \Exp{X_{ij} } = M_{ij}$ and $\Exp{\Ind{E_{ij} = 1} } = \Prob{E_{ij} = 1} = p$.
\end{proof}

\begin{lemma}\label{lem:cov_Y}
When $\star$ is replaced with $0$,
	\begin{align*}
		 \Exp{(Y - pM)^T (Y - pM)}	
		 	&= p(1-p) diag\big( \Exp{X^T X} \big) + p^2 \Exp{ \big(X - M \big)^T \big(X - M \big) }.
	\end{align*}
\end{lemma}
%We denote the covariance by $\Sigma_Y(p) \triangleq \Exp{ (Y^i - pM^i )^T (Y^i - pM^i ) }$.

\begin{proof}
	First of all, we note that 
	\begin{align*}
		\Exp{ (Y - pM )^T (Y - pM ) }
			&= \sum_{i=1}^N  \Exp{ (Y^i - pM^i )^T (Y^i - pM^i ) }.
	\end{align*}
For $(i, j_1, j_2) \in [N] \times [d] \times [d]$,
	\begin{align*}
		\Exp{ (Y^i - pM^i )^T (Y^i - pM^i ) }_{j_1 j_2}
			&= \Exp{ (Y_{ij_1} - pM_{ij_1} ) (Y_{ij_2} - pM_{ij_2} ) }\\
			&= \Exp{ \big(X_{ij_1} \Ind{E_{ij_1} = 1} - pM_{ij_1}\big)\big(X_{ij_2} \Ind{E_{ij_2} = 1} - pM_{ij_2}\big) }\\
			&= \Exp{ X_{ij_1}X_{ij_2}\Ind{E_{ij_1}=1} \Ind{E_{ij_2}=1} + p^2 M_{ij_1} M_{ij_2}}\\
				&\qquad	- \Exp{ p M_{ij_2} X_{ij_1} \Ind{E_{ij_1}=1} + pM_{ij_1} X_{ij_2} \Ind{E_{ij_2}=1}} \\
			&=
				\begin{cases}
					p \Exp{ X_{i j_1}^2 } - p^2 \Exp{ M_{ij_1}^2 }
						& \text{when }j_1 = j_2,\\
					p^2 \Exp{ \big(X_{i j_1} - M_{i j_1} \big) \big(X_{i j_2} - M_{i j_2} \big) } 
						& \text{when }j_1 \neq j_2.
				\end{cases}
	\end{align*}
	
	Recall the identity $ \Exp{ \big(X_{ij_1} - \Exp{X_{ij_1}} \big)^2 } =  \Exp{ X_{ij_1}^2 } -  \Exp{X_{ij_1}}^2 $; 
	we can rewrite
	\begin{align*}
		p \Exp{ X_{i j_1}^2 } - p^2 \Exp{ M_{ij_1}^2  }= p(1-p) \Exp{ X_{i j_1}^2 } 
				+ p^2 \Exp{ \big(X_{ij_1} - M_{ij_1} \big)^2 }.
	\end{align*}
	Therefore,
	\begin{align*}
		\Exp{ (Y^i - pM^i )^T (Y^i - pM^i ) }
			&=p(1-p) diag \big( \Exp{(X^i)^T X^i } \big) 
				+ p^2 \Exp{ \big(X^i - M^i \big)^T \big(X^i - M^i \big) }.
	\end{align*}
	Summing up over $i \in [N]$ concludes the proof.
\end{proof}

%\section{Some Properties of $Y - pM$: Deferred Proofs from Section \ref{sec:missing_values} }\label{sec:properties}

%%% ======================================================================
%%%	Sub-gaussianity of  $Y^i - pM^i$
%%% ======================================================================

\subsection{Sub-gaussianity of  $Y^i - pM^i$ }%\label{sec:proof_prop1}
In this section we show an upper bound on $\| Z^i \|_{\psi_2}$ ($i \in [N]$). For that purpose we recall some definitions 
and known results in Appendix \ref{sec:preliminaries_subG}. Then we state and prove Proposition \ref{prop:subG_Z} 
in \ref{sec:proof_prop1}.

\subsubsection{Preliminaries}\label{sec:preliminaries_subG}
It is common in literature to equip the space of sub-gaussian random variables with the Orlicz $\psi_2$ norm, also known 
as the sub-gaussian norm. Note that $\|X\|_{\psi_2} < \infty$ if and only if $X$ is sub-gaussian.
\begin{definition}[Sub-gaussian norm]\label{defn:psi2}
The sub-gaussian norm of a random variable $X$ is defined as
\[	\| X \|_{\psi_2} \triangleq \inf \bigg\{ t > 0 \text{ such that } \bbE \bigg[ \exp\big( \frac{X^2}{t^2} \Big) \bigg] \leq 2 \bigg\}.	\]
\end{definition}
Imitating the Orlicz $\psi_2$-norm, we define a different type of sub-gaussian norm $\| \cdot \|_{\phi_2} $, which turns out 
to be equivalent to $\| \cdot \|_{\psi_2}$. %For more details about $\| \cdot \|_{\phi_2}$, see Appendix \ref{sec:subG_properties}.

\begin{definition}\label{defn:phi2}
	\[
		\| X \|_{\phi_2} \triangleq \inf\left\{ \tau > 0	\text{ such that } 
			\Exp{e^{\lambda X}} \leq e^{\frac{1}{2}\lambda^2 \tau^2} \text{ for all }\lambda \in \Reals \right\}.
	\]
\end{definition}
We note here that $\| \cdot \|_{\phi_2}$ of a centered Bernoulli random variable is known.
\begin{theorem}[Theorem 2.1 from \citet{buldygin2013sub}]\label{thm:phi2_Bernoulli}
	Let $X(p)$ be a centered Bernoulli random variable with parameter $p \in [0,1]$, i.e., $X(p) = 1-p$ with 
	probability $p$ and $X(p) = -p$ with probability $1-p$. Then
	\begin{equation*}%\label{eqn:Kp}
		\big\| X(p) \big\|_{\phi_2} = \begin{cases}
			0,						& p \in \{ 0, 1 \};	\\
			\frac{1}{4},					& p = \frac{1}{2};	\\
			\frac{2p - 1}{2 \ln \frac{p}{1-p}},	& p \in (0,1) \setminus \{ \frac{1}{2} \}.
	\end{cases}
	\end{equation*}
\end{theorem}
From now on, we reserve $K(p)$ to denote the $\phi_2$ norm of the centered Bernoulli random variable with parameter $p \in [0,1]$, i.e.,
\begin{equation}\label{eqn:Kp.0}
	K(p) \triangleq = \begin{cases}
			0,						& p \in \{ 0, 1 \};	\\
			\frac{1}{4},					& p = \frac{1}{2};	\\
			\frac{2p - 1}{2 \ln \frac{p}{1-p}},	& p \in (0,1) \setminus \{ \frac{1}{2} \}.
	\end{cases}
\end{equation}
Observe that $K(p)$ is concave and maximized at $p = \frac{1}{2}$ with bounded range within $[0, \frac{1}{4}]$ for all $p \in [0, 1]$.

The two norms, $\|\cdot\|_{\psi_2}$ and $\|\cdot\|_{\phi_2}$ are equivalent norms in the sense that there exist absolute constants 
 $c_1, c_2 > 0$ such that $\|X \|_{\psi_2} \leq c_1 \| X \|_{\phi_2}$ and $\|X \|_{\phi_2} \leq c_2 \| X \|_{\psi_2}$ hold for all 
 random variables. We define and utilize the following constants throughout:
\begin{align}
c_1 & \triangleq \inf\{ \gamma > 0: \| X \|_{\psi_2} \leq \gamma \| X \|_{\phi_2} ~\text{for all sub-gaussian }X \}, \label{eq:constant.c1} \\
c_2 & \triangleq \inf\{ \gamma > 0: \| X \|_{\phi_2} \leq \gamma \| X \|_{\psi_2} ~\text{for all sub-gaussian }X \}, \label{eq:constant.c2}
\end{align}
Lastly, we extend the notion of sub-gaussianity to random vectors.
\begin{definition}
	A $d$-dimensional real-valued random vector $X $ is said to be a sub-gaussian vector if $u^T X$ is sub-gaussian 
	for all $u \in \bbS^{d-1}$. Moreover, 
	\[	\| X \|_{\psi_2} \triangleq \sup_{u \in \bbS^{d-1}} \| u^T X \|_{\psi_2} \quad\text{and}\quad 
			\| X \|_{\phi_2} \triangleq \sup_{u \in \bbS^{d-1}} \| u^T X \|_{\phi_2}.	\]
\end{definition}

\subsubsection{An Upper Bound on $\big\| Y^i - pM^i \big\|_{\psi_2}$ }\label{sec:proof_prop1}
\begin{proposition}\label{prop:subG_Z}
Suppose that $\star$ is replaced with $0$ and 
$p \in (0,1]$ is given. For every $i \in [N]$,
\begin{align*}
	\big\| Y^i - pM^i \big\|_{\psi_2} &\leq \big\| X^i - M^i \big\|_{\psi_2} + c_1 (\max_{j \in [d]} | M_{ij}|) K(p) %\\
			%&
			\leq c_1 \big( \tau^\star + K(p) \big)
\end{align*}
where $K(\cdot)$ as defined in \eqref{eqn:Kp.0} and $c_1$ as defined in \eqref{eq:constant.c1}. 
\end{proposition}

\begin{proof}
	Recall that $M$ is a deterministic matrix, even though it is not directly observable. In our model, $Y_{ij} = 
	X_{ij} \Ind{E_{ij} = 1} + \star ~ \Ind{E_{ij} = 0}$ for each $(i, j) \in [N] \times d$; see \eqref{eq:YnX}.
	With $\star$ replaced by $0$, we may write $Y^i = X^i \circ E^i$ for all $i \in [N]$. Here, $E^i \in \Reals^d$ is 
	a random vector with i.i.d. Bernoulli random variables with parameter $p$ and $\circ$ denotes the Hadarmard 
	product (entrywise product). For each $(i,j) \in [N] \times [d]$, we decompose $Y_{ij} - pM_{ij}$ as follows:
	\begin{align*}
		Y_{ij} - pM_{ij} &= X_{ij} \Ind{ E_{ij} = 1} - pM^i\\
			&= \big(X_{ij} - M_{ij} \big) \Ind{E_{ij} = 1} + \big( \Ind{E_{ij} = 1} - p \big) M_{ij}.
	\end{align*}
By definition of $\psi_2$ norm,
	\begin{align*}
		\big\| Y^i - pM^i \big\|_{\psi_2} 
			&= \sup_{u \in \bbS^{d-1}} \big\| \big( Y^i - pM^i  \big) u\big\|_{\psi_2}
			=  \sup_{u \in \bbS^{d-1}} \bigg\| \sum_{j=1}^d \big( Y_{ij} - pM_{ij}  \big) u_j \bigg\|_{\psi_2}.
	\end{align*}
	Fix $u \in \bbS^{d-1}$. By triangle inequality,
	\begin{align}
		\bigg\| \sum_{j=1}^d \big( Y_{ij} - pM_{ij}  \big) u_j \bigg\|_{\psi_2}
			&\leq		\bigg\| \sum_{j=1}^d \big(X_{ij} - M_{ij} \big) \Ind{E_{ij} = 1} u_j \bigg\|_{\psi_2}	\label{eqn:subgZ_term.1}\\
				&\qquad+ \bigg\| \sum_{j=1}^d  \big( \Ind{E_{ij} = 1} - p \big) M_{ij} u_j \bigg\|_{\psi_2}.	\label{eqn:subgZ_term.2}
	\end{align}
	In the rest of the proof, we establish upper bounds on \eqref{eqn:subgZ_term.1} and \eqref{eqn:subgZ_term.2} separately.
	Since the indicator function is nonnegative and upper bounded by $1$, the first term in \eqref{eqn:subgZ_term.1} has 
	the following upper bound: 
	\begin{align*}
		\bigg\| \sum_{j=1}^d \big(X_{ij} - M_{ij} \big) \Ind{E_{ij} = 1} u_j \bigg\|_{\psi_2}
			\leq \bigg\| \sum_{j=1}^d \big(X_{ij} - M_{ij} \big) u_j \bigg\|_{\psi_2}
			\leq \big\| X^i - M^i \big\|_{\psi_2},
	\end{align*}
	which is again bounded above by $c_1\tau^\star$ due to \eqref{eq:subG}.
	The second term in \eqref{eqn:subgZ_term.2} has a trivial upper bound as follows:
	\begin{align*}
		\bigg\| \sum_{j=1}^d  \big( \Ind{E_{ij} = 1} - p \big) M_{ij} u_j \bigg\|_{\psi_2}
			&\leq c_1 \bigg\| \sum_{j=1}^d  \big( \Ind{E_{ij} = 1} - p \big) M_{ij} u_j \bigg\|_{\phi_2}\\
			&\leq c_1 \left\{\sum_{j=1}^d \big\| ~\Ind{E_{ij} = 1} - p ~\big \|_{\phi_2}^2 \big( M_{ij} u_j \big)^2   \right\}^{1/2}\\
			&\leq c_1 \max_{j \in [d]} |M_{ij}| \big\| ~\Ind{E_{ij} = 1} - p ~\big \|_{\phi_2}.
	\end{align*}
	The first inequality is due to the equivalence between $\| \cdot \|_{\psi_2}$ and $\| \cdot \|_{\phi_2}$. The second 
	inequality is based on the property of the sum of independent sub-gaussian random variables, cf. \citet{vershynin2018high}. The last inequality is trivial because $\sum_{j=1}^d M_{ij}^2 u_j^2 = 
	\max_{j \in [d]}|M_{ij}|^2 \|u\|_2^2$ and $\max_{j \in [d]}|M_{ij}| \leq 1$.
	Lastly, we note that $\Ind{E_{ij}} - p$ is a centered Bernoulli random variable with parameter $p$.
	By Theorem \ref{thm:phi2_Bernoulli},
	\[	\big\| \Ind{E_{ij}} - p \big\|_{\phi_2} = K(p)	\]
\end{proof}

\subsection{Bounding $\|Y - pM\|_2$}

\begin{lemma} \label{lem:noise_singular}
With probability at least $1-2/N^4$, 
\begin{align}
	\|Y - pM\|_2 &\leq \sqrt{ \frac{1-p}{4} + p (\tau^\star)^2} \sqrt{Np} \nonumber \\
				&\qquad + c_1 (\tau^\star + K(p))
					\Big\{\sqrt{c_3} \big[ (Nd)^{1/4} + (2N \log N)^{1/4} \big] \vee c_3 \big( \sqrt{d} + \sqrt{2 \log N} \big) \Big\}\label{eq:noise_singular}
\end{align}
Here, $c_3 = c_3(c_0) > 0$ is an absolute constant, which depends  only on another absolute constant $c$\footnote{Though 
we do not specify what the value of $c$ is, it is an absolute constant that does not depend on other parameters. It is the same constant $c$ 
that appears in the exponent of Bernstein's inequality.}%; see Lemma \ref{lem:Bernstein}.}.
\end{lemma}
\begin{proof}
Let $Z = Y - pM$. Observe that $Z$ is an $N \times d$ matrix whose rows $Z^i$ are independent, mean zero, 
sub-gaussian random vectors in $\Reals^d$. 
The Matrix Chernoff bound %(cf. Corollary \ref{cor:op_upper}) 
implies that for every $t \geq 0$,
\begin{align*}
	\sigma_{1}(Y - pM) &\leq \Big\{ \sigma_{1}\big( \Exp{(Y - pM)^T (Y - pM)} \big) \\
		&\qquad+ \max_{i \in [N]}\big\| Y^i - pM^i \big\|_{\psi_2}^2
		\Big[ c_3 \sqrt{N}(\sqrt{d} + t ) \vee c_3^2 (\sqrt{d} + t )^2 \Big] \Big\}^{1/2}
\end{align*}
with probability at least $1 - 2 \cdot 9^{-t^2}$.
By Lemma \ref{lem:cov_Y},
\begin{align*}
	 \Exp{(Y - pM)^T (Y - pM)}	
	 	&= p(1-p) diag\big( \Exp{ M^T M} \big)	+ p^2 \Exp{ \big(X - M \big)^T \big(X - M \big) }.
\end{align*}
Therefore,
\begin{align*}
	&\sigma_{1}\big( \Exp{(Y - pM)^T (Y - pM) } \big)\\
	&\qquad\leq p(1-p) \sigma_1 \Big( diag\big( \Exp{M^T M} \big)\Big)	+p^2 \sigma_1 \Big( \Exp{ \big(X - M \big)^T \big(X - M \big) } \Big)\\
	&\qquad\leq p(1-p) \max_{j \in [d]} \bigg( \sum_{i=1}^N \Exp{X_{ij }^2} \bigg) 	+ p^2 \sum_{i=1}^N \sigma_1 \Big( \Exp{ \big(X^i - M^i \big)^T \big(X^i - M^i \big) } \Big)  \\
	&\qquad\leq Np \Big[ (1-p) \Big( \max_{i,j} \Exp{X_{ij}^2} \Big) + p \max_{i \in [N]} \| X^i \|_{\phi_2}^2 \Big]. 
\end{align*}
The last inequality follows as $\Big\| \Exp{ \big(X^i - M^i \big)^T \big(X^i - M^i \Big) } \Big\|_2 \leq \| X^i \|_{\phi_2}^2$. 
Recall that $|X_{ij}| \leq \frac{1}{2}$ for all $i, j$; $ \| X^i \|_{\phi_2}^2 \leq \tau^\star$ and 
$\big\| Y^1 - pM^1 \big\|_{\psi_2} \leq c_1 (\tau^* + K(p))$ by Proposition \ref{prop:subG_Z}. Using these and 
choosing $t = \sqrt{2 \log N}$ completes the proof as $9^{-t^2} = \big( e^{2 \ln 3}\big)^{-2 \log N} 
= \big(\frac{1}{N^4}\big)^{\ln 3} \leq \frac{1}{N^4}$ for all $N \geq 1$.
\end{proof}
%
%
%\begin{corollary}\label{cor:Z_spectral}
%	Suppose that Conditions \ref{cond:bounded} and \ref{cond:subgaussian} hold. 
%	Then for given $p \in [0,1]$,
%	\begin{align*}
%		\sigma_{1}(Y - pM) 
%			&\leq \sqrt{ (1-p) \Lambda^2 + p \tau^2} \sqrt{Np}\\
%				&\qquad + c_1 (\Lambda + \tau K(p))
%					\Big\{\sqrt{c_3} \big[ (Nd)^{1/4} + (2N \log N)^{1/4} \big] \vee c_3 \big( \sqrt{d} + \sqrt{2 \log N} \big) \Big\}
%	\end{align*}
%	with probability at least $1 - \frac{2}{N^4}$. Here, $c_3 = c_3(c_0) > 0$ is the same constant as in 
%	Lemma \ref{lem:noise_singular}.
%\end{corollary}
%
%\begin{proof}
%First of all, we utilize a simple inequality: for $a, b > 0$, $\sqrt{a + b} \leq \sqrt{a} + \sqrt{b}$. 
%When Conditions \ref{cond:bounded} holds, $\max_{i,j} \big| M_{ij} \big| \leq \Lambda$ and 
%$\big\| Y^1 - pM^1 \big\|_{\psi_2} \leq c_1 (\Lambda + \tau K(p))$ by Proposition \ref{prop:subG_Z}.
%\end{proof}
%
%\begin{remark}\label{rem:E1}
%Suppose that Conditions \ref{cond:bounded} and \ref{cond:subgaussian} hold. Then
%\[	\Prob{\cE_1^c} \leq \frac{2}{N^4}	\]
%follows as an immediate consequence of Corollary \ref{cor:Z_spectral}.
%\end{remark}
%

%%% ======================================================================
%%%	Proof for Mixture Clustering Theorem
%%% ======================================================================
\section{Proof of Theorem \ref{thm:main}}\label{sec:proof_mixture_clustering}

The proof of Theorem \ref{thm:main} is presented in this section. It is divided into four parts. Section \ref{ssec:aux}
presents useful auxiliary properties that will be utilized in establishing result. Section \ref{ssec:high.prob} is about
identifying an event that holds with high-probability and under which the result will naturally follow. 
Section \ref{ssec:full.result} states and establishes a ``pre-cursor'' of Theorem \ref{thm:main} stated
as Theorem \ref{thm:main_clustering_temp}. In Section \ref{ssec:conclude} we argue Theorem \ref{thm:main} as a consequence of
Theorem \ref{thm:main_clustering_temp}.

\subsection{Useful auxiliary properties}\label{ssec:aux}

Given a matrix $A = \sum_{j=1}^{\min\{N, d\}} \sigma_j(A) u_j v_j^T$ 
and a threshold $t \geq 0$, we define operator $\varphi_{A;t}^{\text{hard}}: 
\Reals^d \to \Reals^d$ as 
\begin{equation} \label{eq:hardsvt.1} 
	\varphi^{\text{hard}}_{A; t} : w \mapsto \sum_{j=1}^d \Ind{\sigma_j(A) \geq t}  \cP_{v_j} (w),
\end{equation}
where $\cP_{v_j}$ denotes the projection onto $\text{span}(v_j)$; specifically, when $w \in \Reals^d$ 
as a row vector representation, $\cP_{v_j}(w) = w v_j v_j^T$. If we take the column vector 
representation, it would be natural to consider $\cP^{v_j}(w) = v_j v_j^T w$
\footnote{These are just two different representations of the same projection operation, 
and we shall switch between the row vector representation (left multiplication) and the column vector representation  
(right multiplication) for convenience of explanation; however it will be clear from the context whether we are using row or
column representation.}.

\begin{lemma}\label{lem:contraction}
	For any matrix $A \in \Reals^{N \times d}$ and $t \geq 0$, $\varphi^{\text{hard}}_{A;t}$ is a contraction, i.e.,
	\[	\big\| \varphi^{\text{hard}}_{A;t} (w) \big\|_2 \leq \| w \|_2,	\qquad \forall w \in \Reals^d.	\]
\end{lemma}

\begin{proof}
	Write $A = \sum_{j=1}^{\min\{N, d\}} \sigma_j (A) u_j v_j^T$. Recall that the right singular vectors of $A$ 
	forms an orthonormal basis of $\Reals^{\min\{N, d\}}$. For any $w \in \Reals^d$,
	\begin{align*}
		\big\| \varphi^{\text{hard}}_{A;t} (w) \big\|_2^2
			&= \left(\sum_{j=1}^{\min\{N, d\}} \Ind{\sigma_j(A) \geq t}  \cP_{v_j} (w) \right)^2\\
			&= \sum_{j=1}^{\min\{N, d\}} \Ind{\sigma_j(A) \geq t}  \cP_{v_j} (w)^2\\
			&\leq  \sum_{j=1}^{\min\{N, d\}} \cP_{v_j} (w)^2\\
			&= \| w \|_2^2.
	\end{align*}
\end{proof}
Given a matrix $A = \sum_{j=1}^{\min\{N, d\}} \sigma_j(A) u_j v_j^T$, define 
\begin{align}\label{eq:def.hsvt}
\HSVT_t(A) & = \sum_{j=1}^{\min\{N, d\}} \sigma_j(A) \Ind(\sigma_j(A) \geq t) u_j v_j^T.
\end{align}
We relate $\varphi^{\text{hard}}_{A; t}$ and $\HSVT_t(A)$ next.
\begin{lemma}\label{lem:hsvt_induced}
Given $A \in \Reals^{N \times d}$ and $t \geq 0$,  
	\[	\HSVT_t(A)^i = \varphi^{\text{hard}}_{A; t} \big( A^i \big), \qquad \forall i \in [N].	\]
\end{lemma}

\begin{proof}  
Let $A = \sum_{j = 1}^{\min\{N, d\}} \sigma_j(A) u_j v_j^T$ be the SVD of $A$. By the orthonormality of the left singular vectors, 
\begin{align*}
	 \varphi^{\text{hard}}_{A; t} \big( A^i \big)
	 	&= \sum_{a=1}^{\min\{N, d \}} \Ind{\sigma_i (A) \geq t} \cP_{v_i} \big(A^a \big)
		= \sum_{a=1}^{\min\{N, d \}} \Ind{\sigma_a (A) \geq t}A^i v_a v_a^T\\
		&= \sum_{a, b=1}^{\min\{N, d \}} \Ind{\sigma_a (A) \geq t} e_i^T \sigma_{b}(A) u_{b} v_{b}^T  v_a v_a^T\\
		&= e_i^T \sum_{a, b=1}^{\min\{N, d \}} \Ind{\sigma_a (A) \geq t} \sigma_{b}(A) u_{b} \delta_{a b} v_a^T\\
		&= e_i^T \HSVT_t(A)
		= \HSVT_t(A)^i.
\end{align*}
This completes the proof. 
\end{proof}

\begin{lemma}\label{lem:subg_projection}
	Let $X \in \Reals^d$ be a sub-gaussian random vector with $\| X \|_{\phi_2} < \infty$. 
	Let $V \in \Reals^{d \times r}$ be a matrix with orthogonal columns. $V$ is allowed to be random as long as it is independent of $X$. 
	Then for any $s > 0$,
	\[	\Prob{ \| VV^T X \|_2^2 > s }	\leq 2r \exp \left( - \frac{s}{2r \|  X \|_{\phi_2}^2 } \right).	\]
\end{lemma}
\begin{proof}
	Observe that $ \| VV^T X \|_2^2 = X^T VV^T X = \sum_{j=1}^r X^T V_j V_j^T X$ where $V_j$ denotes the $j$-th column of $V$.
	Now observe that for any $s > 0$, 
	\begin{align*}
		X^T VV^T X > s		\quad\text{implies that}\quad
				&\exists j \in [r]	\quad\text{s.t.}\quad		X^T V_j V_j^T X > \frac{s}{r}.
	\end{align*}
	This observation leads to the following inequality by applying the union bound:
	\begin{align*}
		\Prob{ X^T VV^T X > s	}	&\leq \sum_{j=1}^r \Prob{X^T V_j V_j^T X > \frac{s}{r}}.
	\end{align*}

	It remains to find an upper bound for
	\begin{align*}
		\Prob{X^T V_j V_j^T X > \frac{s}{r}}
			&= 	\Prob{ V_j^T X > \sqrt{\frac{s}{r}} } + \Prob{ V_j^T X < - \sqrt{\frac{s}{r}} }.
	\end{align*}
	By Markov's inequality (Chernoff bound) and the definition of $\| \cdot \|_{\phi_2}$ (see Definition \ref{defn:phi2}), 
	for any $\lambda > 0$,
	\begin{align*}
		\Prob{ V_j^T X > \sqrt{\frac{s}{r}} }
			&\leq		\Exp{ \exp\left( \lambda   V_j^T X   \right) } \exp\left(- \lambda \sqrt{\frac{s}{r}} \right)\\
			&\leq		\Exp{\Exp{ \exp\left( \lambda   V_j^T X   \right) ~ \big| ~ V }} \exp\left(- \lambda \sqrt{\frac{s}{r}} \right)\\
			&\leq 	\Exp{ \exp\left( \frac{\lambda^2 \|  V_j^T X  \|_{\phi_2}^2}{2} \right) } \exp\left(- \lambda \sqrt{\frac{s}{r}} \right)\\
			&\leq 	\exp\left( \frac{\lambda^2 \|  X  \|_{\phi_2}^2}{2} \right) \exp\left(- \lambda \sqrt{\frac{s}{r}} \right).
	\end{align*}
	Optimizing over $\lambda > 0$ (observe that $-\lambda a + \frac{\lambda^2}{2b}$ is minimized to $- \frac{a^2 b}{2}$ with $\lambda = ab$),
	\begin{align*}
		\Prob{  V_j^T X   > \sqrt{\frac{s}{r}}  }
			&\leq \exp \left( - \frac{s}{2r \|  X \|_{\phi_2}^2 } \right).
	\end{align*}
	We obtain the same upper bound on $\Prob{  V_j^T X  < - \sqrt{\frac{s}{r}}  }$.
	
	All in all,
	\[	\Prob{ X^T VV^T X > s	}	\leq 2r \exp \left( - \frac{s}{2r \| X \|_{\phi_2}^2 } \right).	\]
\end{proof}

%%%%

\subsection{Leave-One-Out Analysis: A High-probability Event }\label{ssec:high.prob}

Recall that $\hM = \HSVT_{\sigma_r(Y)}(Y)$ as per \eqref{eq:mhat} and \eqref{eq:def.hsvt}. In what follows, we shall  
denote threshold $t^* = \sigma_r(Y)$, i.e. $\hM = \HSVT_{t^*}(Y)$. We define an event $\cE$
\begin{align}
	 \cE &:= \Bigg\{ \forall i \in [N],~~ \left\| \varphi_{Y;t^*}^{\text{hard}} \big(  Y^i - p M^i \big) \right\|_2^2 
					\leq  4 \big\| Y^i - p M^i \big\|_{\phi_2}^2  r \log(2dN^5)	\nonumber\\
				&\qquad\qquad\qquad\qquad\qquad%\qquad\qquad\qquad
					\times\Bigg[ 1 + \frac{2 d^2 \log(2dN^5)}{r}\left(  \frac{  \big\| Y^i - p M^i \big\|_{\phi_2} }{p \sigma_r(M) - 3 \| Y - p M \|_2} \right)^2 \Bigg]  \Bigg\}.
					\label{eqn:cE_3p}
\end{align}
In the following lemma, we show this event also has high probability.

\begin{lemma}\label{lem:ell2_upper_leave_one_out}
Suppose that $p > \frac{ 3 \| Y - pM \|_2}{\sigma_r(M)}$. Then % and $\| Y - pM \|_2 < t^* < \sigma_r(M) - \| Y - pM \|_2$. Then 
\[	\Prob{\cE} \leq \frac{2}{N^4}.		\]
\end{lemma}

\begin{proof} We establish the proof in four steps. 
\paragraph{Step 1: Decomposition of $Y$ by `Leave-one-row-out'.}
Fix $i \in [N]$. We define two matrices $Y^{\circ}, Y^{\times} \in \Reals^{N \times d}$ so that $Y = Y^{\circ} + Y^{\times}$:
\[
	\big(Y^{\circ}\big)^j = 
		\begin{cases}
			Y^j	&	\text{if } j = i,\\
			0	&	\text{if } j \neq i.
		\end{cases}
	\qquad\text{and}\qquad
	\big(Y^{\times}\big)^j = 
		\begin{cases}
			0	&	\text{if } j = i,\\
			Y^j	&	\text{if } j \neq i.
		\end{cases}
\]
That, $Y^\circ$ contains only $i$th row of $Y$ and rest $0$s while $Y^\times$ contains all but $i$th row of $Y$ with $i$th row being $0$s.

Let $v_1^Y, \ldots, v_r^Y$ and $v_1^{\times}, \ldots, v_r^{\times}$ denote the top $r$ right singular vectors of $Y$ and $Y^{\times}$, respectively. 
Also, we note that $\varphi_{Y;t^*}^{\text{hard}} \big(  Y^i - p M^i \big) = \cP_{\vspan{v_1^Y, \ldots, v_r^Y}} \big(  Y^i - p M^i \big)$ and rewrite it as 
\begin{align*}
	\varphi_{Y;t^*}^{\text{hard}} \big(  Y^i - p M^i \big)
		&=  \cP_{\vspan{v_1^Y, \ldots, v_r^Y}} \big(  Y^i - p M^i \big)\\
		& =  \cP_{\vspan{v_1^{\times}, \ldots, v_r^{\times}}} \big(  Y^i - p M^i \big) \\
			&\quad+ \Big[ \cP_{\vspan{v_1^Y, \ldots, v_r^Y}} \big(  Y^i - p M^i \big) - \cP_{\vspan{v_1^{\times}, \ldots, v_r^{\times}}} \big(  Y^i - p M^i \big) \Big].
\end{align*}
Then Young's inequality for products yields
\begin{align}
	\left\| \varphi_{Y;t^*}^{\text{hard}} \big(  Y^i - p M^i \big) \right\|_2^2 
		&\leq 2\left\|  \cP_{\vspan{v_1^{\times}, \ldots, v_r^{\times}}} \big(  Y^i - p M^i \big) \right\|_2^2	\nonumber\\
			&\quad+ 2 \left\| \cP_{\vspan{v_1^Y, \ldots, v_r^Y}} \big(  Y^i - p M^i \big) - \cP_{\vspan{v_1^{\times}, \ldots, v_r^{\times}}} \big(  Y^i - p M^i \big) \right\|_2^2	\nonumber\\
		&\leq 2\left\| \cP_{\vspan{v_1^{\times}, \ldots, v_r^{\times}}} \big(  Y^i - p M^i \big)\right\|_2^2		\label{eqn:leave_one_out.1}\\
			&\quad+  2\left\|  \cP_{\vspan{v_1^Y, \ldots, v_r^Y}} - \cP_{\vspan{v_1^{\times}, \ldots, v_r^{\times}}} \right\|_2^2 \left\|  Y^i - p M^i \right\|_2^2.	\label{eqn:leave_one_out.2}
\end{align}

\paragraph{Step 2: Finding an upper bound on \eqref{eqn:leave_one_out.1}.}
Recall that the rows in $Y$ are independent. Therefore, the row vector of interest, $Y^i - pM^i$, is independent of the matrix $Y^{\times}$, 
and hence, is independent of the subspace spanned by $\{ v_1^{\times}, \ldots, v_r^{\times} \}$. By applying Lemma \ref{lem:subg_projection}, 
we observe that for any $s_1 > 0$,
\begin{align*}
	&\Prob{\left\| \mathcal{P}_{\vspan{v_1^{\times}, \ldots, v_r^{\times}} } \big(  Y^i - p M^i \big) \right\|_2^2 > s_1}
		\leq 2 r \exp \left( - \frac{s_1}{2r \left\|  Y^i - p M^i  \right\|_{\phi_2}^2} \right).
\end{align*}
Choosing $s_1 = 2r \left\|  Y^i - p M^i  \right\|_{\phi_2}^2 \log ( 2r N^5)$ leads to
\begin{equation}\label{eqn:leave_one_out.3}
	\Prob{\left\| \mathcal{P}_{\vspan{v_1^{\times}, \ldots, v_r^{\times}} } \big(  Y^i - p M^i \big) \right\|_2^2 
		> 2r \left\|  Y^i - p M^i  \right\|_{\phi_2}^2 \log ( 2r N^5)}
		\leq \frac{1}{N^5}.
\end{equation}

\paragraph{Step 3: Finding an upper bound on \eqref{eqn:leave_one_out.2}.}
Next, we apply Davis-Kahan $\sin \Theta$ theorem  to find an upper bound on 
$\left\|  \cP_{\vspan{v_1^Y, \ldots, v_r^Y}} - \cP_{\vspan{v_1^{\times}, \ldots, v_r^{\times}}} \right\|_2$.
Recall that $Y = pM + (Y - pM)$. Since $\rank(M) = r$, $\sigma_r(M) > \sigma_{r+1}(M) = 0$. By Weyl's theorem, 
\begin{align*}
	\sigma_r(Y) &\geq \sigma_r(M) - \| Y - pM \|_2.
\end{align*}
Similarly, since $Y = Y^{\times} + Y^{\circ}$,
\begin{align*}
	\sigma_{r+1}(Y^{\times}) &\leq \sigma_{r+1}(Y) + \sigma_1(Y^{\circ}) \leq \sigma_{r+1}(M) + \| Y - pM \|_2 + \| Y^{\circ} \|_2\\
		&= \| Y - pM \|_2 + \| Y^i - p M^i \|_2\\
		&\leq 2 \| Y - pM \|_2,
\end{align*}
Therefore, if $p > \frac{3 \| Y - pM \|_2}{\sigma_r(M)}$, then $\sigma_r(Y) > \sigma_{r+1}( Y^{\times} )$. Then by Davis-Kahan theorem,
\begin{align*}
	\left\|  \cP_{\vspan{v_1^Y, \ldots, v_r^Y}} - \cP_{\vspan{v_1^{\times}, \ldots, v_r^{\times}}} \right\|_2
		&\leq  \frac{\| Y^{\circ} \|_2}{ \sigma_r(Y) - \sigma_{r+1}( Y^{\times} ) }\\
		&\leq \frac{\| Y^i - pM^i \|_2}{ p \sigma_r(M) - 3 \| Y - pM \|_2 }.
\end{align*}
Again by applying Lemma \ref{lem:subg_projection}, for any $s_2 > 0$,
\begin{align*}
	&\Prob{\left\|  Y^i - p M^i \right\|_2^2 > s_2}
		\leq 2 d \exp \bigg( - \frac{s_2}{2d \left\|  Y^i - p M^i  \right\|_{\phi_2}^2} \bigg).
\end{align*}
Choosing $s_2 = 2d \left\|  Y^i - p M^i  \right\|_{\phi_2}^2 \log ( 2d N^5)$ leads to
\begin{equation}\label{eqn:leave_one_out.4}
	\Prob{\left\|  Y^i - p M^i \right\|_2^2 > 2d \left\|  Y^i - p M^i  \right\|_{\phi_2}^2 \log ( 2d N^5)}
		\leq \frac{1}{N^5}.
\end{equation}

\paragraph{Step 4: Concluding the proof.}
Putting \eqref{eqn:leave_one_out.1}, \eqref{eqn:leave_one_out.2}, \eqref{eqn:leave_one_out.3}, \eqref{eqn:leave_one_out.4} together, 
we obtain the following: for each $i \in [N]$,
\begin{align*}
	\left\| \varphi_{Y;t^*}^{\text{hard}} \big(  Y^i - p M^i \big) \right\|_2^2
		&\leq 4 \big\| Y^i - p M^i \big\|_{\phi_2}^2 \bigg[ r \log(2r N^5) + \frac{2 \| Y^i - pM^i \|_{\phi_2}^2 d^2 \log^2(2dN^5)}{( p \sigma_r(M) - 3 \| Y - pM \|_2 )^2} \bigg]
\end{align*}
with probability at least $1 - \frac{2}{N^5}$. The conclusion follows by taking union bound over $i \in [N]$.
Lastly, observe that $\| Y^i - pM^i \|_{\phi_2} \leq c_1 c_2 \big(\tau^\star +  K(p) \big)$ by Proposition \ref{prop:subG_Z}.
\end{proof}

\subsection{A Pre-cursor of \ref{thm:main} and Its Proof}\label{ssec:full.result}

Here we present a more generic version of Theorem \ref{thm:main} from which the statement of Theorem \ref{thm:main}
will follow immediately. 

\begin{theorem}[Full Version of Theorem \ref{thm:main}]\label{thm:main_clustering_temp}
Let $p > \frac{3  \|Y - pM \|_2}{\sigma_r(M)}$. Then with probability at least $1 - \frac{2}{N^4}$, 
the following inequalities hold for all $i_1, i_2 \in [N]$:
\begin{itemize}
\item when $\alpha(i_1) = \alpha(i_2)$:
\begin{align*}
	\big\| \widehat{M}^{i_1} - \widehat{M}^{i_2} \big\|_2 
	&  \leq 2c_1 c_2  ( \tau^\star+ K(p))  \sqrt{ r \log(2dN^5)}	
			\left[ 1 + \frac{d \sqrt{2\log(2dN^5)}}{\sqrt{r}}  \frac{  c_1 c_2 \big( \tau^\star +  K(p) \big) }{p \sigma_r(M) - 3  \|Y - pM \|_2}  \right] .
\end{align*}
\item when $\alpha(i_1) \neq \alpha(i_2)$:
\begin{align*}
	\big\| \widehat{M}^{i_1} - \widehat{M}^{i_2} \big\|_2 
		&\geq \left( 1 - \bigg( \frac{ \| Y - pM \|_2 }{p \sigma_r(M) - \| Y - pM \|_2 }\bigg)^2\right)^{\frac12} \big\| p M^{i_1} - pM^{i_2} \big\|_2\\
		&\quad - 4c_1 c_2  ( \tau^\star+ K(p))  \sqrt{ r \log(2dN^5)}	
			\left[ 1 + \frac{d \sqrt{2\log(2dN^5)}}{\sqrt{r}}  \frac{  c_1 c_2 \big( \tau^\star +  K(p) \big) }{p \sigma_r(M) - 3  \|Y - pM \|_2}  \right].
\end{align*}
\end{itemize}
\end{theorem}
\begin{proof}
We establish the Theorem statement under condition $\cE$. As argued in Lemma \ref{lem:ell2_upper_leave_one_out}, $\Prob{\cE^c} \leq 2/N^4$. Therefore, the statement of theorem holds with probability at least $1-2/N^4$. To that end, we shall assume $\cE$ holds. 
Now, for any $i_1, i_2 \in [N]$
\begin{align}
	\widehat{M}^{i_1} - \widehat{M}^{i_2}
		&= \varphi_{Y;t^*}^{\text{hard}} \big(Y^{i_1}\big) - \frac{1}{\hat{p}} \varphi_{Y;t^*}^{\text{hard}} \big(Y^{i_2} \big)	\nonumber\\
		&=  \varphi_{Y;t^*}^{\text{hard}} \big(Y^{i_1} - pM^{i_1} \big) -  \varphi_{Y;t^*}^{\text{hard}} \big(Y^{i_2} - pM^{i_2} \big)
			+ \varphi_{Y;t^*}^{\text{hard}} \big(pM^{i_1} - pM^{i_2} \big).  	\label{eqn:clustering.term.0}
\end{align}
\paragraph{For $i_1, i_2$ with $\alpha(i_1) = \alpha(i_2)$:}
$M^{i_1} = M^{i_2}$. By triangle inequality, \eqref{eqn:clustering.term.0} yields
\begin{equation}\label{eqn:clustering.term.1}
	\big\| \widehat{M}^{i_1} - \widehat{M}^{i_2} \big\|_2
		\leq  \big\| \varphi_{Y;t^*}^{\text{hard}} \big(Y^{i_1} - pM^{i_1} \big) \big\|_2 
			+ \big\| \varphi_{Y;t^*}^{\text{hard}} \big(Y^{i_2} - pM^{i_2} \big) \big\|_2
\end{equation}
When conditioned $\cE$, for each $i \in [N]$,
\begin{align}
	\left\| \varphi_{Y;t^*}^{\text{hard}} \big(  Y^i - p M^i \big) \right\|_2 
		&\leq  \big\| Y^i - p M^i \big\|_{\phi_2}  \sqrt{ r}  \sqrt{\log(2dN^5)}
		\nonumber\\
			&\qquad\times\left[ 1 + \frac{2 d^2 \log(2dN^5)}{r}\left(  \frac{  \big\| Y^i - p M^i \big\|_{\phi_2} }{p \sigma_r(M) - 3 \|Y - pM \|_2} \right)^2 \right]^{1/2}
			\nonumber\\
		&\leq  c_1 c_2  ( \tau^\star+ K(p))  \sqrt{ r \log(2dN^5)}	\nonumber\\
			&\qquad\times\left[ 1 + \frac{2 d^2 \log(2dN^5)}{r}\left(  \frac{  c_1 c_2 \big( \tau^\star +  K(p) \big) }{p \sigma_r(M) - 3  \|Y - pM \|_2} \right)^2 \right]^{1/2}.		\label{eqn:clustering.term.2}
\end{align}
The first inequality follows from $\cE$ and second inequality follows from Proposition \ref{prop:subG_Z} and \eqref{eq:constant.c2}.
Amalgamating \eqref{eqn:clustering.term.1} with \eqref{eqn:clustering.term.2} and further simplifying by the fact that 
$\sqrt{a+b} \leq \sqrt{a} + \sqrt{b}$ for $a, b > 0$ yields the first conclusion.

\paragraph{For $i_1, i_2$ with $\alpha(i_1) \neq\alpha(i_2)$:}
Observe that $M^i = \mu^{\alpha(i)}$. Applying triangle inequality to \eqref{eqn:clustering.term.0}, we have
\begin{align}
	\big\| \widehat{M}^{i_1} - \widehat{M}^{i_2} \big\|_2
		&\geq \bigg\| \varphi_{Y;t^*}^{\text{hard}} \big(pM^{i_1} - pM^{i_2} \big) \bigg\|_2	\nonumber\\
			&\quad-  \big\| \varphi_{Y;t^*}^{\text{hard}} \big(Y^{i_1} - pM^{i_1} \big) \big\|_2 
			-  \big\| \varphi_{Y;t^*}^{\text{hard}} \big(Y^{i_2} - pM^{i_2} \big) \big\|_2.	\label{eqn:clustering.term.3}
\end{align}
The (twice and negative of) upper bound in \eqref{eqn:clustering.term.2} will provide lower bound last two terms of \eqref{eqn:clustering.term.3}. For
the first term, $\bigg\| \varphi_{Y;t^*}^{\text{hard}} \big(M^{i_1} - M^{i_2} \big) \bigg\|_2$, let $\alpha(i_1) = a \neq \alpha(i_2) = b$ with $a, b \in [r]$; and
hence $M^{i_1} = \mu^a, ~M^{i_2} = \mu_b$.  Let $v_1^M, \ldots, v_r^M$ denote the top-$r$ right singular vectors of $M$. Then we observe that 
$\mu^a - \mu^b \in \text{span}\{ v_1^M, \ldots, v_r^M\}$. Therefore, 
\[	\cP_{\text{span}\{v_1^M, \ldots, v_r^M\}} \big(  \mu^a - \mu^b \big) =  \mu^a - \mu^b.	\]
Since $t^* = \sigma_r(Y)$, $\varphi_{Y;t^*}^{\text{hard}} \big(  \mu^a - \mu^b \big) = \cP_{\text{span}\{v_1^Y, \ldots, v_r^Y\}}\big(  \mu^a - \mu^b \big)$, 
and by Pythagorean theorem,
\begin{align}
	\left\| \varphi_{Y;t^*}^{\text{hard}} \big(  \mu^a - \mu^b \big) \right\|_2^2
		&= \big\|  \mu^a - \mu^b \big\|_2^2 - \left\| \cP_{\text{span}\{v_1^Y, \ldots, v_r^Y\}}\big(  \mu^a - \mu^b \big) 
			- \big(  \mu^a - \mu^b \big) \right\|_2^2		\nonumber\\
		&= \big\|  \mu^a - \mu^b \big\|_2^2 - \left\| \Big( \cP_{\text{span}\{v_1^Y, \ldots, v_r^Y\}}\big(  \mu^a - \mu^b \big) 
			- \cP_{\text{span}\{v_1^M, \ldots, v_r^M\}} \Big) \big(  \mu^a - \mu^b \big) \right\|_2^2		\nonumber\\
		&\geq \left( 1 - \bigg( \frac{ \| Y - pM \|_2 }{p \sigma_r(M) - \| Y - pM \|_2 }\bigg)^2\right)
			\big\|  \mu^a - \mu^b \big\|_2^2.	\label{eqn:mixture_subg.term.4}
\end{align}
In above, the last inequality follows again using Davis-Kahan Sin $\Theta$ Theorem. Combining \eqref{eqn:mixture_subg.term.4}
and \eqref{eqn:clustering.term.2} in \eqref{eqn:clustering.term.3} concludes the proof.
\end{proof}

\subsection{Concluding Proof of Theorem \ref{thm:main}}\label{ssec:conclude}

\begin{lemma}\label{lem:penultimate.1}
Let $p \geq \frac{4\| Y - pM \|_2 }{\sigma_r(M)}$ and 
\begin{align}\label{eq:separation}
p \Gamma  & \geq 9 c_1 c_2  ( \tau^\star+ K(p))  \sqrt{ r \log(2dN^5)}	
			\left[ 1 + \frac{d \sqrt{2\log(2dN^5)}}{\sqrt{r}}  \frac{  4c_1 c_2 \big( \tau^\star +  K(p) \big) }{p \sigma_r(M)}  \right].
\end{align}
Then all points are clustered correctly with probability at least $1-2/N^4$ using distance threshold 
\begin{align}\label{eq:dist.threshold}
\threshold & = 2c_1 c_2  ( \tau^\star+ K(p))  \sqrt{ r \log(2dN^5)}	
			\left[ 1 + \frac{d \sqrt{2\log(2dN^5)}}{\sqrt{r}}  \frac{ 4 c_1 c_2 \big( \tau^\star +  K(p) \big) }{p \sigma_r(M)}  \right]
\end{align}
using the algorithm described in Section \ref{sec:algorithm}.
\end{lemma}
\begin{proof}
Let $p \geq \frac{4\| Y - pM \|_2 }{\sigma_r(M)}$. That is, $p \sigma_r(M) \geq 4 \| Y - pM \|_2$. Then, 
\begin{align}
\left( 1 - \bigg( \frac{ \| Y - pM \|_2 }{p \sigma_r(M) - \| Y - pM \|_2 }\bigg)^2\right)^{\frac12}  & \geq \left(1-\bigg(\frac{1}{3}\bigg)^2\right)^{\frac12} %\nonumber \\
%& 
= \sqrt{\frac{8}{9}} ~>~\frac{2}{3}.
\end{align} 
Using this inequality in the statement of Theorem \ref{thm:main_clustering_temp}, we obtain that if 
\begin{align}\label{eq:desire.1}
\big\| p M^{i_1} - pM^{i_2} \big\|_2 & \geq 9 c_1 c_2  ( \tau^\star+ K(p))  \sqrt{ r \log(2dN^5)}	
			\left[ 1 + \frac{d \sqrt{2\log(2dN^5)}}{\sqrt{r}}  \frac{  c_1 c_2 \big( \tau^\star +  K(p) \big) }{p \sigma_r(M) - 3  \|Y - pM \|_2}  \right] 
\end{align}
then clustering of all $N$ points happens successfully by selecting distance threshold as defined in \eqref{eq:dist.threshold} in
the distance clustering algorithm in Section \ref{sec:algorithm}. We further note that since $p\sigma_r(M) \geq 4 \|Y-pM\|_2$, $p\sigma_r(M) - 3 \|Y-pM\|_2 \geq p\sigma_r(M)/4$. Using this inequality along with the notation 
$\Gamma = \min_{a, b \in [k]} \|\mu^a - \mu^b \|_2$, condition  \eqref{eq:desire.1} implies \eqref{eq:separation}. This completes the proof. 
\end{proof}
Next, we quantify a high-probability bound on $\|Y - pM\|_2$ to conclude proof of Theorem \ref{thm:main} using Lemma \ref{lem:penultimate.1}.

Now we evaluate bound on $\| Y - pM\|_2$ using Lemma \ref{lem:noise_singular}. Recall from bound from 
\eqref{eq:noise_singular}. In our setting, $d = n^2$ and $\log N = o(n) = o(\sqrt{d})$. Therefore, 
\begin{align}
\big[ (Nd)^{1/4} + (2N \log N)^{1/4} \big] \vee c_3 \big( \sqrt{d} + \sqrt{2 \log N} \big) & = C_1 (n + \sqrt{n}N^{\frac14}).
\end{align}
For the other term, notice that 
\begin{align}
 \sqrt{ (1-p) + p (\tau^\star)^2} \sqrt{p} & \leq \sqrt{p + p^2 (\tau^\star)^2} ~\leq \sqrt{p} + p \tau^\star,
 \end{align}
 where we have used $\sqrt{a + b} \leq \sqrt{a} + \sqrt{b}$ for any $a, b \geq 0$.
 Therefore, by defining $\Delta$ as in \eqref{eqn:Delta_upper} with large enough absolute constants, we obtain that 
 $\|Y - pM\|_2 \leq \Delta$.

Now replacing $d = n^2$ and hence noticing that $\log(2dN^5) = O(\log nN)$, multiplying $\sqrt{r}$
inside the big term in right hand side of \eqref{eq:separation}, bringing out $\sqrt{\log (2dN^5)}$ term
and bounding it by $O(\sqrt{\log nN})$ and defining appropriately large enough constant $C'$, 
\eqref{eq:separation} is implied by 
\begin{align}
p \Gamma  & \geq C' ( \tau^\star+ K(p))  \log(nN) \Big( \sqrt{r} + \frac{n^2 \big( \tau^\star +  K(p) \big) }{p \sigma_r(M)}  \Big).
\end{align}
This completes the proof of Theorem \ref{thm:main}.

\section{Proof of Corollary \ref{cor:main}}
\begin{proof}
To complete the proof, it suffices to verify that the conditions of Theorem \ref{thm:main} are satisfied. 
We show that if $p_k = 1/r$ for all $k \in [r]$; $s(n) = \omega(\sqrt{r/n})$; $N = n^{4}$; and 
$p \geq \frac{C \tau^\star \sqrt{r} \log n}{\Gamma} $ for a sufficiently large constant $C > 0$, then
\begin{enumerate}
	\item
	$\log N = o(n)$,
	\item
	$p > \frac{4 \Delta}{\sigma_r(M)}  \vee \frac{16 \log(Nd)}{\min(N, d)}$, and 
	\item
	the inequality in \eqref{eq:gap.condition} is satisfied.
\end{enumerate}

Notice that the choice of $p, N$ per \eqref{eq:cor.main} already satisfies basic conditions such as $\log N = o(n)$, and 
$p \gg \frac{\log Nn}{\min(N, n^2)}$ since $\Gamma = O(n)$, $\tau^\star \geq 1$ and $\log Nn = \log \big(n^5\big) = \Theta(\log n)$. 

Given this, the key conditions for us to verify are 
\begin{enumerate}
	\item
	$p > \frac{4 \Delta}{\sigma_r(M)}$.
	\item
	The gap condition in \eqref{eq:gap.condition}:
	\begin{align}\label{eq:gap.condition.1}
	\Gamma  & \geq C' \frac{( \tau^\star+ K(p))  \log(nN)}{p} \bigg( \sqrt{r} + \frac{n^2 \big( \tau^\star +  K(p) \big) }{p \sigma_r(M)}  \bigg).
	\end{align}
\end{enumerate}
We verify these conditions one by one.

\medskip
\noindent{\em 1. Verifying $p > \frac{4 \Delta}{\sigma_r(M)}$.} %Note that it suffices to argue that $p\sigma_r(M) > 4 \Delta$; see 
Recall \eqref{eqn:Delta_upper} for the definition of $\Delta$. From \eqref{eq:singular.values},  $\sigma_r(M) 
= s(n) \sqrt{Nn^2/r}$. Since $N = n^4$, we have $\sigma_r(M) = \Theta(s(n) n^3 / \sqrt{r})$. It follows from the choice 
$p  \geq \frac{C \tau^\star \sqrt{r} \log n}{\Gamma} $ that
\begin{align}\label{eq:bd.1}
p \sigma_r(M) & \geq \frac{\tilde{C} \tau^\star s(n) n^3 \log n}{\Gamma} ~\geq~\tilde{C}' \tau^\star n^2 \log n
\end{align}
because $\Gamma = O\big(n  s(n) \big)$. Here, $\tilde{C}, \tilde{C}'$ are constants. 

Next, we simplify the terms in $\Delta$ to show $p \sigma_r(M) > 4 \Delta$.
First, we observe from \eqref{eqn:Kp.0} that $K(p) \leq 1$. Also, we note that $N = n^4 \geq n^2$ and $\tau^\star \geq 1$. 
Therefore, $\tau^\star + K(p) \leq 2 \tau^\star$ and $n + \sqrt{n} N^{\frac14} \leq 2 \sqrt{n} N^{\frac14}$. By definition of $\Delta$ (cf. see \eqref{eqn:Delta_upper}), we have
\begin{align*}
	\Delta & = C \big((\sqrt{p} + p \tau^\star)\sqrt{N} + (\tau^\star + K(p)) ( n + \sqrt{n} N^{\frac14})\Big)
		\leq C \sqrt{p N} + C p \tau^* \sqrt{N} + 4\tau^\star \sqrt{n}N^{\frac14}.
\end{align*}

Now we argue that $p \sigma_r(M)/12 \geq \max \big\{ C \sqrt{p N}, ~ C p \tau^* \sqrt{N}, ~ 4\tau^\star \sqrt{n}N^{\frac14} \big\}$ 
to conclude the desired result. 

First, $\frac{1}{12} p\sigma_r(M)$ vs $C \sqrt{pN}$. Recall that $N = n^4$ and that $p \leq 1$, hence, $\sqrt{p} \geq p$.
\begin{align*}
	\frac{\frac{1}{12}p\sigma_r(M)}{C \sqrt{pN}} = \frac{\sqrt{p} \sigma_r(M) }{12C\sqrt{N}}
	\stackrel{(a)}{\geq}\frac{p \sigma_r(M)}{12C\sqrt{N}}  \stackrel{(b)}{\geq} \frac{\tilde{C}}{12C\sqrt{N}} \tau^\star n^2 \log n 
	\stackrel{(c)}{\geq} \frac{\tilde{C}}{12C\sqrt{N}} n^2 \log n \geq 1
\end{align*}
when $n$ is sufficiently large. Note that (a) follows from $p \leq 1$, (b) follows from \eqref{eq:bd.1}, and (c) follows from $\tau^\star \geq 1$. 

Next, $\frac{1}{12} p\sigma_r(M)$ vs $C p \tau^\star\sqrt{N}$. By Proposition \ref{prop:subg_rum}, we know that $\tau^\star = O(\sqrt{n})$. 
Also, recall from \eqref{eq:singular.values} that $ \sigma_r(M)  = \Theta \big( s(n) n^3 /\sqrt{r} \big)$. From the assumption that 
$s(n) = \omega(\sqrt{r/n})$, we achieve
\begin{align*}
	\frac{\frac{1}{12}p\sigma_r(M)}{C p \tau^\star\sqrt{N}}\geq \frac{ \sigma_r(M) }{12 C' \sqrt{n N } }
	= \frac{ s(n) n^3}{ 12 C'' \sqrt{r n^5}}
	= \omega \big( 1 \big)
\end{align*}
and therefore, $\frac{1}{12} p\sigma_r(M) \geq C p \tau^\star\sqrt{N}$ when $n$ is sufficiently large.

Finally, $\frac{1}{12} p\sigma_r(M)$ vs $4 \tau^\star \sqrt{n}N^{\frac14}$. Since $p  \sigma_r(M)  \geq \tilde{C}' \tau^\star n^2 \log n$ 
by \eqref{eq:bd.1}, it follows that
\begin{align*}
	\frac{\frac{1}{12}p\sigma_r(M)}{4 \tau^\star \sqrt{n}N^{\frac14}}
	= \frac{p \sigma_r(M)}{48\tau^\star n^{1.5}} 
	& \geq \frac{  \tilde{C}' \tau^\star n^2 \log n }{48 \tau^\star n^{1.5}}
	\geq C'' \sqrt{n} \log n, 
\end{align*}
where we have used that $\Gamma = O \big( s(n) n \big)$ in the last inequality. This implies that $\frac{1}{12}p\sigma_r(M) 
\geq 4 \tau^\star \sqrt{n}N^{\frac14}$ when $n$ is sufficiently large. 

Collectively this completes the verification of $p > \frac{4 \Delta}{\sigma_r(M)}$.

\medskip
\noindent{\em 2. Verifying the inequality in \eqref{eq:gap.condition.1}.} Note that $K(p) \leq 1$, $\tau^\star \geq 1$. 
For the choice of $N = n^4$ and $p = \frac{C \tau^\star \sqrt{r} \log n}{\Gamma}$, the right hand side of \eqref{eq:gap.condition.1}

\begin{align*}
	C' \frac{( \tau^\star+ K(p))  \log(nN)}{p} \bigg( \sqrt{r} + \frac{n^2 \big( \tau^\star +  K(p) \big) }{p \sigma_r(M)}  \bigg)
		&\leq C' \frac{ 2\tau^{\star} \log(n^5)}{p} \bigg( \sqrt{r} + \frac{2 n^2 \tau^\star}{p \sigma_r(M)} \bigg)\\
		&\leq \frac{10 C' \Gamma}{C \sqrt{r}} \Bigg( \sqrt{r} + \frac{2 n^2 \tau^\star}{ \frac{C \tau^\star \sqrt{r} \log n}{\Gamma} \frac{C'' s(n) n^3}{\sqrt{r}}} \Bigg)\\
		&= \frac{10 C' \Gamma}{C \sqrt{r}} \Bigg( \sqrt{r} + \frac{2 \Gamma }{ C C'' s(n) n \log n } \Bigg).
\end{align*}
Since $\Gamma = O \big( n s(n) \big)$, $ \frac{2 \Gamma }{ C C'' s(n) n \log n } = O \big( \frac{1}{\log n} \big)$. Therefore, when $n$ is 
sufficiently large, $\frac{2 \Gamma }{ C C'' s(n) n \log n } \leq \sqrt{r}$. Lastly, we observe that if $C > 20 C''$, then 
\[	\frac{10 C' \Gamma}{C \sqrt{r}} \Bigg( \sqrt{r} + \frac{2 \Gamma }{ C C'' s(n) n \log n } \Bigg) \leq \frac{20 C' }{C} \Gamma < \Gamma.  \]
This completes verification of \eqref{eq:gap.condition}.
\end{proof}

%%% ======================================================================
%%%	Proof in Ranking Section
%%% ======================================================================

\section{Proof of Proposition \ref{prop:subg_rum}}\label{sec:proof_ranking}

\begin{proof} Here we wish to establish that if $\perm$ is generated as per a random utility model, 
then the distribution of $\iota(\perm)$, the embedding into pair-wise comparisons, is sub-Gaussian 
with parameter $\sqrt{n}$. Precisely, we shall establish that 
	\[	\big\| \iota(\perm) - \bbE \iota(\perm) \big\|_{\psi_2}\leq C \sqrt{n-1}.	\]
To that end, observe that $\perm = \perm(\eps_1, \ldots, \eps_n; u_1, \ldots, u_n)$ is a random permutation, 
which depends on $n$ independent sources of randomness, $\eps_1, \ldots, \eps_n$ and $n$ deterministic 
(but potentially hidden) parameters, $u_1, \ldots, u_n$. So is $\iota(\perm)$, which is an 
${n \choose 2}$-dimensional binary random vector.
For any $v \in \bbS^{{n \choose 2} - 1}$, we define a random function $f_v$ as
\begin{align*}
	f_v( \eps_1, \ldots, \eps_n)
		\triangleq v^T \iota\big(\perm(\eps_1, \ldots, \eps_n) \big).
\end{align*}
By definition \ref{def:rum}, $\perm$ is generated from $n$ independent random variables 
$\eps_1, \ldots, \eps_n$ and $n$ deterministic parameters $u_1, \ldots, u_n$. The change in 
the value of $\eps_i$, results in the change in $Z_i$, and possibly affects $\perm$. However, the change in $\eps_i$ 
can affect at most $(n-1)$ pairwise comparisons out of ${n \choose 2}$ when all the other random variables in 
$\{ \eps_j \}_{j \neq i}$ remain unchanged.
 
Since $\iota(\perm) \in \big\{ \pm \frac{1}{2} \big\}^{n \choose 2}$, we can observe that for any $i \in [n]$,
\begin{align*}
	\sup_{\eps_1, \ldots \eps_n \atop \eps_i'} 
		\Big| f_v(\eps_1, \ldots, \eps_n) - f_v (\eps_1, \ldots, \eps_{i-1}, \eps_i', \eps_{i+1}, \ldots, \eps_n) \Big|
	\leq \sum_{j \neq i} \big| v_{ij} \big|.&
\end{align*}
The bounded difference concentration inequality yields
\begin{equation}\label{eqn:bounded_conc}
	\Prob{ f_v(\perm) - \bbE_{\perm} f_v(\perm) > t } \leq e^{-t^2/2\nu},	
\end{equation}
where the the variance is bounded from above by the Efron-Stein upper bound as
\begin{align}\label{eqn:efron_stein}
	\nu	&= \frac{1}{4} \sum_{i=1}^n \left( \sum_{j \neq i} \big| v_{ij} \big| \right)^2
		\leq  \frac{1}{4} \sum_{i=1}^n \left( (n-1) \sum_{j \neq i} \big| v_{ij} \big|^2 \right)
		= \frac{n-1}{2} \sum_{1 \leq i < j \leq n} v_{ij}^2
		= \frac{n-1}{2}.	
\end{align}
Inserting \eqref{eqn:efron_stein} to \eqref{eqn:bounded_conc}, we obtain the following inequality: for any $t > 0$,
\[	\Prob{ f_v(\perm) - \bbE_{\perm} f_v(\perm) > t } \leq e^{-\frac{t^2}{n-1}},	
		\quad\text{for all }v \in \bbS^{{n \choose 2} - 1}.	\]
The conclusion follows from the equivalence between the tail probability characterization and the Orlicz function 
characterization.% (see 1 and 4 in Proposition \ref{prop:subG_characterization}).
%$\| \iota(\perm) - \bbE\iota(\perm) \|_{\psi_2} \leq C \sqrt{n-1}$ for some absolute constant $C > 0$.
\end{proof}

\section{Proof of Proposition \ref{prop:mixture_mnl}}

\begin{proof}
Condition 2 (random ordering) implies that for $a \neq b \in [k]$, 
\begin{align*}
	\| \bbE X^a - \bbE X^b \|_2^2 
		&= \sum_{1 \leq i < j \leq n} \big( \Prob{ \sigma_a^{-1}(i) \geq \sigma_a^{-1}(j) } - \Prob{ \sigma_b^{-1}(i) \geq \sigma_b^{-1}(j) } \big)^2.
\end{align*}
First, we observe that
\begin{align*}
	&\bbE_{\sigma_a, \sigma_b} \| \bbE X^a - \bbE X^b \|_2^2 \\
		&\qquad\geq \bbE_{\sigma_a, \sigma_b} \Bigg[ \sum_{1 \leq i < j \leq n} 
			\Ind{ \big( \sigma_a^{-1}(i) - \sigma_a^{-1}(j) \big) \big( \sigma_b^{-1}(i) - \sigma_b^{-1}(j) \big) < 0 }
			\bigg( \frac{ 1 }{1 + \exp \big( - \frac{\rho}{\beta} \big)} -  \frac{ \exp \big( -\frac{\rho}{\beta} \big) }{1 + \exp \big( - \frac{\rho}{\beta} \big)} \bigg)^2	\Bigg]\\
		&\qquad=\frac{1}{2} {n \choose 2} \bigg( \frac{ 1 - \exp \big( - \frac{\rho}{\beta} \big)}{ 1 + \exp \big( - \frac{\rho}{\beta} \big)}\bigg)^2.
\end{align*}
Then by the usual concentration argument (e.g., convex Lipschitz function; cf. \citet[Theorem 7.12]{boucheron2013concentration}), we can show that for any $t \geq 0$,
\begin{align*}
	\Prob{  \| \bbE X^a - \bbE X^b \|_2  - \bbE_{\sigma_a, \sigma_b} \| \bbE X^a - \bbE X^b \|_2  > t } \leq 2 \exp \left( - \frac{t^2}{4n} \right).
\end{align*}
Recall that $\bbE_{\sigma_a, \sigma_b} \| \bbE X^a - \bbE X^b \|_2^2 \geq \big( \bbE_{\sigma_a, \sigma_b} \| \bbE X^a - \bbE X^b \|_2 \big)^2$. 
Therefore, with probability at least $1 - \frac{2}{n^4}$,
 \[
 	\| \bbE X^a - \bbE X^b \|_2 \geq \frac{\sqrt{n(n-1)}}{2}  \frac{ 1 - \exp \big( - \frac{\rho}{\beta} \big)}{ 1 + \exp \big( - \frac{\rho}{\beta} \big)} - 4 \sqrt{n \log n}.
\]
Note that ${r \choose 2} = r(r-1)/2 \leq r^2/2$. Applying the union bound on all pairs $a, b \in [r]$, we can conclude the proof.
\end{proof}

%%% ==========================================
%%%	Example 2: Gaussian RUM
%%% ==========================================

\section{Proof of Proposition \ref{prop:mixture_gaussian}}
\begin{proof}
Observe that
\begin{align*}
	\big\| \bbE X^a - \bbE X^b \big\|_2^2
		&= \sum_{1 \leq i < j \leq n } \Big( \bbE X_{(i,j)}^a - \bbE X_{(i,j)}^b \Big)^2\\
		&= \sum_{1 \leq i < j \leq n } \Big(  \Prob{ \eps_i^a - \eps_j^a \geq u_j^a - u_i^a } - \Prob{ \eps_i^b - \eps_j^b \geq u_j^b - u_i^b } \Big)^2 \\
		&= \sum_{1 \leq i < j \leq n } \Bigg( \Phi\bigg( \frac{ u_j^a - u_i^a }{\sigma \sqrt{2}} \bigg) 
				- \Phi \bigg( \frac{ u_j^b - u_i^b }{\sigma \sqrt{2}} \bigg)  \Bigg)^2.
\end{align*}
%{\color{red} In the last line, $x \vee y \triangleq \max\{x, y\}$ and $x \wedge y \triangleq \min\{x, y \}$ for $x, y \in \Reals$.}

Note that for each $a \in [r]$, 
\[
	u_j^a - u_i^a = 	\begin{cases}	
			1	&	\text{with probability }\frac{1}{4},\\
			0	&	\text{with probability }\frac{1}{2},\\
			-1	&	\text{with probability }\frac{1}{4}.
		\end{cases}
\]
Also, $u_j^a - u_i^a$ and $u_j^b - u_i^b$ are independent for $a \neq b \in [k]$. Therefore, for $a, b \in [k]$ and $i, j \in [n]$,
\[
	\Bigg|\, \Phi\bigg( \frac{ u_j^a - u_i^a }{\sigma \sqrt{2}} \bigg) 
				- \Phi \bigg( \frac{ u_j^b - u_i^b }{\sigma \sqrt{2}} \bigg) \, \Bigg|
		= \begin{cases}
			\Phi \Big( \frac{1}{\sigma \sqrt{2}} \Big) - \Phi \Big( - \frac{1}{\sigma \sqrt{2}} \Big)	&	\text{with probability }\frac{1}{8},\\
			\Phi \Big( \frac{1}{\sigma \sqrt{2}} \Big) - \Phi( 0 )	&	\text{with probability }\frac{1}{2},\\
			0	&	\text{with probability }\frac{3}{8}.
		\end{cases}
\]
Then, 
\begin{equation}\label{eqn:expectation_gaussian}
	\bbE_{u^a, u^b} \big\| \bbE X^a - \bbE X^b \big\|_2^2 
		= \frac{1}{4} {n \choose 2} \bigg( \Phi \Big( \frac{1}{\sigma \sqrt{2}} \Big) - \Phi \Big( - \frac{1}{\sigma \sqrt{2}} \Big) \bigg)^2.
\end{equation}

By the same concentration argument as in the proof of Proposition \ref{prop:mixture_mnl}, we can show that for any $t \geq 0$, 
\begin{align*}
	\Prob{  \| \bbE X^a - \bbE X^b \|_2  - \bbE_{u^a,  u^b} \| \bbE X^a - \bbE X^b \|_2  > t } \leq 2 \exp \left( - \frac{t^2}{4n} \right).
\end{align*}
 That is to say, for each $a \neq b \in [r]$,
 \begin{equation}\label{eqn:concentration.aa}
 	\| \bbE X^a - \bbE X^b \|_2 \geq \bbE_{u^a, u^b} \| \bbE X^a - \bbE X^b \|_2 - 4 \sqrt{n \log n}
\end{equation}
with probability at least $1 - \frac{2}{n^4}$. 

Note that $\bbE_{u^a, u^b} \big\| \bbE X^a - \bbE X^b \big\|_2^2  \geq \big( \bbE_{u^a, u^b} \big\| \bbE X^a - \bbE X^b \big\|_2 \big)^2$. 
We conclude the proof by inserting the lower bound on 
$\bbE_{u^a, u^b} \| \bbE X^a - \bbE X^b \|_2$ from \eqref{eqn:expectation_gaussian} into \eqref{eqn:concentration.aa} 
and applying the union bound for all pairs of $a, b \in [r]$. 
%The simplified lower bound follows from \eqref{eqn:gaussian_lower.1} and \eqref{eqn:gaussian_lower.2}. 
 \end{proof}

\end{document}